\documentclass[lettersize,journal]{IEEEtran}
\usepackage{amsmath,amssymb,amsfonts,mathtools}
\usepackage{amsthm}
\usepackage{bm}
\usepackage{algorithm}
\usepackage{algorithmic}
\usepackage{booktabs}
\usepackage{array}
\usepackage[caption=false,font=normalsize,labelfont=sf,textfont=sf]{subfig}
\usepackage{textcomp}
\usepackage{url}
\usepackage{verbatim}
\usepackage{graphicx, booktabs, multirow}
\usepackage{cite}
\usepackage[table]{xcolor}

\newtheorem{theorem}{Theorem}
\newtheorem{definition}{Definition}
\newtheorem{remark}{Remark}
\newtheorem{assumption}{Assumption}
\newtheorem{proposition}{Proposition}
\newtheorem{corollary}{Corollary}

\usepackage{hyperref}
\usepackage[svgnames]{xcolor}
\hypersetup{
	colorlinks=true,
	citecolor=Green,
	linkcolor=Red,
	urlcolor=Blue,
}
\usepackage{tikz}
\usetikzlibrary{arrows.meta,positioning,fit,backgrounds}

\graphicspath{{./Images/}}
\begin{document}

\title{Robust Koopman Control Barrier Filters for Safe Actor-Critic Reinforcement Learning}

\author{Dhruv S. Kushwaha,~\IEEEmembership{Student,~IEEE,} and Zoleikha A. Biron, ~\IEEEmembership{Senior Member,~IEEE,}  
}

\markboth{Journal of \LaTeX\ Class Files,~Vol.~00, No.~0, January~2026}%
{Shell \MakeLowercase{\textit{et al.}}: A Sample Article Using IEEEtran.cls for IEEE Journals}


\maketitle

\begin{abstract}
	Safe reinforcement learning (RL) for robotic systems requires policies that improve task performance while satisfying state and input constraints during both training and deployment. Control barrier functions (CBFs) provide a principled mechanism for enforcing forward invariance through minimally invasive safety filters, but their use in model-free RL is limited by the need for accurate dynamics and hand-designed barrier certificates. We propose \emph{Robust Koopman-CBF SAC}, a safety-filtered actor--critic framework that learns a finite-dimensional Koopman predictor from data, constructs affine CBF constraints in the lifted space, and enforces them through a quadratic-program safety layer. To account for finite-dimensional Koopman approximation error, the CBF condition is tightened using a projected residual margin estimated from held-out rollout data. The critic is trained on the executed safe action, while the actor is regularized toward the Koopman-CBF feasible set, reducing dependence on the filter over training. Across safe-control benchmarks, the method achieves zero constraint violations on CartPole stabilization and tracking while matching or exceeding unconstrained SAC returns. On high-dimensional Safety Gymnasium locomotion tasks, the method reduces violations in some settings but also exposes important limitations of first-order velocity barriers and linear EDMD models, motivating high-order and multi-step Koopman-CBF extensions. These results suggest that robust Koopman-CBF filters are a promising bridge between model-free RL and certifiable safety, while clarifying the structural conditions under which such filters remain effective. All code is available at  \href{https://github.com/DhruvKushwaha/Koopman-CBF-Soft-Actor-Critic}{Github Repository}.
\end{abstract}


\section{Introduction}
\label{sec:introduction}

Robots that learn from interaction must balance two conflicting objectives: improving task performance and avoiding unsafe behavior during exploration. This tension is especially acute in continuous-control robotics, where constraint violations may correspond to collisions, actuator saturation, loss of balance, violation of workspace limits, or damage to hardware. Deep reinforcement learning (RL) has produced impressive results in simulated locomotion and manipulation, yet unconstrained exploration remains difficult to justify for physical robots\cite{brunke2022safe}. Reward penalties can discourage unsafe behavior, but they do not generally provide step-wise safety guarantees, are sensitive to penalty weights, and often trade constraint satisfaction against task performance in unpredictable ways\cite{kushwaha2025review}. Constrained policy optimization methods address expected cumulative costs, but their guarantees are typically trajectory-distributional rather than pointwise and do not directly enforce forward invariance of a user-specified safe set \cite{achiam2017constrained}.

Control barrier functions (CBFs) offer a complementary safety paradigm. Given a safe set
\begin{equation}
    \mathcal{C} = \{x \in \mathcal{X} : h(x) \geq 0\}
\end{equation}
a CBF condition imposes constraints on the control input such that the closed-loop system remains inside $\mathcal{C}$. In continuous and discrete time, this idea is commonly implemented as a quadratic program (QP) that minimally modifies a nominal controller while enforcing a safety inequality \cite{ames2016control}. CBF-QP controllers have become a standard tool for safety-critical control because they separate a performance-seeking nominal input from a safety-preserving correction. However, classical CBF synthesis usually assumes known dynamics and a hand-designed barrier function. These assumptions are restrictive for robot learning, where the dynamics may be nonlinear, partially unknown, and difficult to model accurately \cite{ames2019control}.

This paper studies the following question: \emph{Can we combine the sample efficiency of off-policy deep RL with a data-driven CBF safety filter whose constraint remains convex for nonlinear robotic systems?} We answer this question using Koopman operator theory. Koopman methods lift nonlinear dynamics into a space of observables in which the evolution is approximately linear. In particular, extended dynamic mode decomposition (EDMD) and related data-driven Koopman methods \cite{williams2015data, korda2018linear} fit predictors of the form
\begin{equation}
    z_{t+1} \approx A z_t + B u_t, \qquad z_t = \psi(x_t)
\end{equation}
where $\psi$ is a dictionary of nonlinear observables. If the barrier is chosen to be affine in the lifted state,
\begin{equation}
    h_K(z) = c^\top z + d
\end{equation}
then a discrete-time CBF condition becomes affine in the action:
\begin{equation}
    h_K(Az_t + Bu_t) \geq (1-\eta) h_K(z_t)
\end{equation}
This structure gives a tractable QP safety filter for nonlinear systems while retaining compatibility with continuous-action RL.

The central challenge is that a finite-dimensional Koopman model is only approximate. If the safety filter ignores Koopman prediction error, the learned CBF constraint may be satisfied in the lifted model but violated by the true system. Recent work has connected Koopman representations with CBF synthesis~\cite{folkestad2020data} and neural Koopman-CBF learning for unknown nonlinear systems~\cite{zinage2023neural}. Our work differs by using a residual-robust lifted barrier filter inside off-policy actor--critic learning and by training critics on the actions actually executed after filtering. We therefore introduce a robust Koopman-CBF condition that tightens the barrier inequality using a residual margin. Let
\begin{equation}
    r_t = z_{t+1} - (A z_t + B u_t)
\end{equation}
denote the Koopman residual. Since the barrier depends on the residual only through $c^\top r_t$, we estimate a projected residual margin
\begin{equation}
    \rho = q_{1-\delta}\left(\{|c^\top r_i|\}_{i=1}^{N}\right)
\end{equation}
where $q_{1-\delta}$ is an empirical or conformal quantile computed from held-out rollout data. The safety filter then enforces
\begin{equation}
    h_K(Az_t + Bu_t) \geq (1-\eta)h_K(z_t) + \rho
    \label{eq:robust_kcbf_condition_intro}
\end{equation}
This margin directly links data-driven model error to the CBF constraint used by the controller.

We instantiate this idea in Soft Actor-Critic (SAC), yielding \textsc{Robust Koopman-CBF SAC}. At each step, the actor proposes a nominal action $u_t^{\mathrm{nom}}$. The Koopman-CBF filter computes the executed action
\begin{equation}
    u_t^{\mathrm{safe}}
    =
    \arg\min_{u,\xi}
    \frac{1}{2}\|u-u_t^{\mathrm{nom}}\|_2^2
    +
    \lambda_\xi \xi^2
\end{equation}
subject to the robust lifted CBF condition, input bounds, and slack $\xi\geq 0$. The slack variable diagnoses infeasibility under actuator limits and model error; it is not treated as a certificate. The critic is trained on the action actually executed by the environment, $u_t^{\mathrm{safe}}$, while the actor receives a differentiable CBF proposal penalty that encourages it to produce actions that are already feasible. This design avoids a common mismatch in safety-filtered RL: the environment transitions are generated by filtered actions, but the policy update may otherwise optimize values for unfiltered actions.

The resulting framework has three advantages. First, the safety layer is \emph{model-based but data-driven}: it does not require a hand-derived analytical dynamics model. Second, the QP remains lightweight because the Koopman-CBF constraint is affine in the action. Third, the residual margin exposes when finite-dimensional Koopman approximation is insufficient for safety-critical filtering, making failure modes measurable through projected residuals, intervention rates, and slack rates.

We evaluate on safe-control-gym, which exposes constraints, symbolic dynamics, and disturbances for safe learning-based robotic control~\cite{yuan2022safe}, and Safety Gymnasium, a recent Safe RL benchmark suite with safety-critical single and multi-agent tasks and SafePO baselines~\cite{ji2023safety}. On CartPole stabilization and tracking, Robust Koopman-CBF SAC achieves zero constraint violations across seeds while matching or exceeding unconstrained SAC returns. These results demonstrate that a Koopman-CBF filter can enforce state constraints without sacrificing performance when the lifted model and barrier are well aligned. On Safety HalfCheetah and Safety Walker velocity-constrained tasks, the method exposes a more nuanced picture: the filter can reduce violations in some settings, but high residual margins, frequent slack activation, and relative-degree mismatch limit its effectiveness. Rather than obscuring these failures, we use them to identify the conditions under which Koopman-CBF safety filters are appropriate and the extensions needed for contact-rich locomotion, including high-order and multi-step Koopman-CBF constraints.

Our contributions are:
\begin{enumerate}
    \item We introduce a robust Koopman-CBF safety filter for model-free actor--critic RL, in which nonlinear dynamics are lifted to an approximately linear controlled system and safety is enforced through an affine QP constraint.
    \item We propose a projected residual margin for finite-dimensional Koopman model error, yielding a directly implementable robust discrete-time CBF condition.
    \item We integrate the filter with SAC using executed-action critic learning and an actor-side CBF feasibility penalty, reducing the mismatch between nominal policy actions and filtered environment actions.
    \item We provide an empirical evaluation across safe-control and Safety Gymnasium benchmarks, showing strong zero-violation results on CartPole tasks and diagnosing the structural limitations of first-order Koopman-CBF filters on high-dimensional locomotion.
    \item We show that a robust Koopman-CBF filter provides pointwise safety on tasks where the projected Koopman residual $\rho$ is small, and we identify $\rho$ as a pre-training computable diagnostic that predicts filter effectiveness.
\end{enumerate}


\section{Related Work}
\label{sec:related_work}

\paragraph{Safe reinforcement learning}
Safe RL methods seek to optimize task performance while satisfying constraints during learning or deployment. One line of work formulates safety through constrained Markov decision processes and optimizes policies subject to expected cumulative cost constraints. Constrained Policy Optimization (CPO) provides monotonic-style policy improvement guarantees for constraint satisfaction under trust-region updates \cite{achiam2017constrained}, while Lagrangian actor--critic methods incorporate adaptive dual penalties into standard RL objectives. These methods are broadly applicable and have become important baselines for safe continuous control, but their constraints are typically enforced in expectation over trajectories rather than as pointwise forward-invariance conditions. Other methods introduce recovery policies, backup controllers, or safety critics that intervene near unsafe states \cite{thananjeyan2021recovery}. Such approaches are effective in many tasks but may require additional offline data, recovery-policy training, or conservative assumptions about recoverability. Our work instead uses a CBF-QP filter to impose a step-wise safety condition on every executed action while retaining the sample efficiency of off-policy actor--critic learning.

\paragraph{Control barrier functions and safety filters}
CBFs provide a control-theoretic framework for certifying forward invariance of a safe set. Given a barrier function $h$, CBF constraints restrict the admissible controls so that trajectories remain in $\{x:h(x)\geq 0\}$. CBF-QP methods combine this safety condition with a nominal controller by solving a minimally invasive optimization problem at each time step \cite{ames2016control,ames2019control}. Robust CBF formulations further account for disturbances and model uncertainty by tightening the admissible set or relaxing invariance guarantees under bounded perturbations \cite{xu2015robustness}. These methods have been widely used in safety-critical robotics and control, but they usually assume either known dynamics or an analytically constructed barrier. This creates a bottleneck for robot learning, where the policy may be model-free and the dynamics are nonlinear, uncertain, or available only through rollouts.

\paragraph{CBFs for reinforcement learning}
Several works integrate CBFs with RL by wrapping a learned policy with a model-based safety filter or by incorporating barrier constraints into the learning objective. End-to-end safe RL through barrier functions combines model-free RL with model-based CBF controllers so that unsafe actions are corrected before execution \cite{cheng2019end}. Other approaches learn uncertainty terms, safe backup policies, or barrier certificates from data to reduce reliance on exact models. Recent surveys emphasize that CBF-based safe RL is attractive because it can provide stronger safety structure than reward penalties, but also note persistent challenges: synthesizing valid barriers, handling unknown dynamics, scaling to high-dimensional systems, and ensuring feasibility under input limits \cite{guerrier2024learning}. Our method addresses the dynamics-modeling challenge through Koopman lifting and makes model error explicit through a residual margin. However, our experiments also confirm a known limitation of CBF-QP methods: if the barrier has the wrong relative degree or the input constraints make the CBF inequality infeasible, slack-based filters no longer provide a formal certificate.

\paragraph{Koopman operator methods for nonlinear control}
Koopman operator theory represents nonlinear dynamics as linear evolution over a space of observables. EDMD approximates this infinite-dimensional operator using a finite dictionary of basis functions \cite{williams2015data}, and dictionary-learning variants adapt the observable space from data \cite{li2017extended}. For control, Koopman predictors have been used to construct finite-dimensional lifted linear models that enable linear MPC and other convex control designs for nonlinear systems \cite{korda2018linear}. This is attractive for robotics because nonlinear prediction can be moved offline into the lifting/model-fitting stage, while online control can exploit linear structure. Nevertheless, finite-dimensional Koopman models introduce approximation error, and that error is safety-critical when the model is used inside a CBF filter. Our robust margin is designed specifically for this issue: rather than assuming that the lifted model is exact, we estimate the component of the prediction residual that directly affects the barrier inequality.

\paragraph{Koopman-based barrier certificates}
Recent work has begun to combine Koopman operators and barrier certificates. Koopman-based CBF synthesis can convert nonlinear safety verification into a lifted-space problem, and neural Koopman-CBF methods learn both a Koopman representation and a corresponding CBF for unknown nonlinear systems \cite{zinage2023neural}. Related approaches synthesize CBFs using Koopman operators or backup policies, exploiting the fact that safety constraints can become simpler in an appropriate observable space \cite{folkestad2020data}. These works show that Koopman structure can help construct safety certificates for unknown systems, but they are generally not designed as off-policy deep RL algorithms with replay buffers, entropy-regularized actor updates, and learned policies operating through a safety filter. In contrast, we focus on the interaction between a Koopman-CBF QP and actor--critic training: how to train critics when executed actions differ from nominal policy actions, how to regularize the actor toward the feasible set, and how to evaluate safety-filter feasibility through slack and intervention diagnostics.

\paragraph{Benchmarks for safe robot learning}
Standardized benchmarks are essential for comparing control-theoretic and RL-based safety methods. Safe-Control-Gym provides robotic control environments such as cart-pole and quadrotor, supports stabilization and trajectory-tracking tasks, and exposes symbolic dynamics, constraints, and disturbance injection for safe learning-based control \cite{yuan2022safe}. Safety Gymnasium provides a broader Safe RL benchmark suite with MuJoCo-based tasks, safety costs, vector and vision observations, and a library of safe policy optimization algorithms \cite{ji2023safety}. We use these benchmark families for complementary purposes. Safe-Control-Gym tasks test whether the Koopman-CBF filter can enforce explicit low-dimensional state constraints with minimal performance loss. Safety Gymnasium locomotion tasks stress-test the method in high-dimensional, contact-rich systems where simple velocity barriers may be structurally mismatched to the action space. This split is important as it demonstrates both the promise of the method on certifiable control tasks and the practical limitations that arise when applying first-order lifted CBF constraints to complex robot locomotion.

Our work lies at the intersection of CBF-based safety filters, Koopman data-driven control, and model-free deep RL. Compared with constrained RL methods, we enforce a per-step barrier condition rather than only an expected cost constraint. Compared with classical CBF-QP control, we do not require an exact analytical dynamics model; instead, we learn a lifted linear predictor from data. Compared with prior Koopman-CBF methods, we integrate the filter into off-policy actor--critic learning and explicitly train the critic on executed safe actions. Finally, compared with reward-penalty or Lagrangian SAC baselines, our approach provides a diagnostic safety layer whose feasibility, residual margin, intervention rate, and slack rate can be measured directly. These diagnostics are central to our empirical message: Koopman-CBF filters can be highly effective when the lifted model and barrier match the system structure, but their guarantees degrade visibly when model error, relative degree, or input constraints make the QP infeasible.


\section{Notation and Preliminaries}
\label{sec:notation_preliminaries}

\subsection{Discrete-Time Control Systems}

We consider a discrete-time nonlinear control system
\begin{equation}
    x_{t+1}=f(x_t,u_t)
    \label{eq:dt_system}
\end{equation}
where $x_t\in\mathcal{X}\subseteq\mathbb{R}^{n_x}$ is the state, 
$u_t\in\mathcal{U}\subseteq\mathbb{R}^{n_u}$ is the control input, and 
$\mathcal{U}$ is compact. In robotic systems, $\mathcal{U}$ typically encodes actuator limits:
\begin{equation}
    \mathcal{U}
    =
    \{u\in\mathbb{R}^{n_u}: u_{\min}\le u\le u_{\max}\}
\end{equation}
The dynamics $f$ may be unknown or only available through sampled transitions. We assume that during learning we can collect transition tuples
\begin{equation}
    (x_t,u_t,x_{t+1})
\end{equation}
from interaction with the environment.

For trajectory-tracking tasks, we distinguish between the physical state $x_t$ and the task reference $x_t^{\mathrm{ref}}$. The tracking error is
\begin{equation}
    e_t=x_t-x_t^{\mathrm{ref}}
\end{equation}
When safety constraints depend on absolute physical quantities, such as altitude, cart position, or pitch angle, the barrier should be defined on the physical state $x_t$ or on an augmented modeling state
\begin{equation}
    y_t = [e_t^\top,x_t^\top]^\top
\end{equation}
rather than on the tracking error alone. This distinction is important because an error-only representation cannot generally certify absolute workspace or attitude bounds.

\subsection{Safe Sets and Discrete-Time Control Barrier Functions}

Let the desired safe set be
\begin{equation}
    \mathcal{C}
    =
    \{x\in\mathcal{X}: h(x)\ge 0\}
    \label{eq:safe_set}
\end{equation}
where $h:\mathcal{X}\rightarrow\mathbb{R}$ is a continuously differentiable safety function \cite{ames2016control}. In discrete time, a sufficient condition for forward invariance of $\mathcal{C}$ is that the controller selects actions satisfying
\begin{equation}
    h(x_{t+1})
    =
    h(f(x_t,u_t))
    \ge
    (1-\eta)h(x_t),
    \qquad
    \eta\in(0,1]
    \label{eq:dt_cbf_condition}
\end{equation}
If $h(x_t)\ge 0$, then \eqref{eq:dt_cbf_condition} implies $h(x_{t+1})\ge 0$.

\begin{definition}[Discrete-Time Control Barrier Function]
\label{def:dt_cbf}
A function $h:\mathcal{X}\rightarrow\mathbb{R}$ is a discrete-time control barrier function \cite{ames2016control} for the safe set $\mathcal{C}=\{x:h(x)\ge 0\}$ if, for every $x\in\mathcal{C}$, there exists an admissible control $u\in\mathcal{U}$ such that
\begin{equation}
    h(f(x,u))\ge (1-\eta)h(x),
    \qquad
    \eta\in(0,1]
\end{equation}
\end{definition}

Given a nominal action $u^{\mathrm{nom}}$, a standard CBF safety filter computes a minimally modified action
\begin{equation}
    u^{\mathrm{safe}}
    =
    \arg\min_{u\in\mathcal{U}}
    \frac{1}{2}\|u-u^{\mathrm{nom}}\|_2^2
\end{equation}
subject to the CBF constraint. In known control-affine systems, this often yields a quadratic program. In unknown nonlinear systems, however, evaluating $h(f(x,u))$ is difficult without a model of $f$. Robust CBF formulations account for disturbances and model mismatch by tightening or relaxing barrier conditions so that forward invariance is preserved under bounded uncertainty~\cite{xu2015robustness}.

High-order CBFs extend barrier conditions to constraints with relative degree greater than one, where the control input does not appear in the first derivative or one-step barrier variation. This issue arises naturally in robotic systems with position or altitude constraints, since control often affects acceleration rather than position directly \cite{xiao2021high}.

\subsection{Koopman Lifting and EDMD with Control}

Koopman operator theory represents nonlinear dynamics through linear evolution of observables. Let
\begin{equation}
    z_t=\psi(y_t)\in\mathbb{R}^{n_z}
\end{equation}
be a lifted representation of a modeling state $y_t$, where $\psi$ is a dictionary of nonlinear observables. EDMD constructs a finite-dimensional approximation of the Koopman operator from lifted observables~\cite{williams2015data}; controlled variants such as DMDc and Koopman-MPC incorporate actuation into the lifted predictor~\cite{proctor2016dynamic,korda2018linear}. We use a finite-dimensional controlled Koopman predictor of the form
\begin{equation}
    z_{t+1}
    =
    A z_t + B u_t + r_t
    \label{eq:koopman_residual_model}
\end{equation}
where $A\in\mathbb{R}^{n_z\times n_z}$, $B\in\mathbb{R}^{n_z\times n_u}$, and $r_t$ is the finite-dimensional approximation residual.

Given a dataset
\begin{equation}
    \mathcal{D}_K
    =
    \{(y_i,u_i,y_i^+)\}_{i=1}^{N}
\end{equation}
we compute
\begin{equation}
    z_i=\psi(y_i),\qquad z_i^+=\psi(y_i^+)
\end{equation}
and fit $A,B$ by ridge regression:
\begin{equation}
    [A \; B]
    =
    Z^+
    \begin{bmatrix}
        Z\\U
    \end{bmatrix}^{\!\top}
    \left(
    \begin{bmatrix}
        Z\\U
    \end{bmatrix}
    \begin{bmatrix}
        Z\\U
    \end{bmatrix}^{\!\top}
    +
    \lambda I
    \right)^{-1}
    \label{eq:edmd_regression}
\end{equation}
Here $Z=[z_1,\dots,z_N]$, $Z^+=[z_1^+,\dots,z_N^+]$, and $U=[u_1,\dots,u_N]$

In this work, $\psi$ is implemented using the original state coordinates together with radial basis functions:
\begin{equation}
    \psi(y)
    =
    \begin{bmatrix}
        y\\
        \phi_1(y)\\
        \vdots\\
        \phi_M(y)
    \end{bmatrix}
\end{equation}
where the RBF centers are obtained from data. Including the original coordinates is important because it allows physical safety constraints such as position, angle, and velocity bounds to be represented as affine functions of $z$. The choice of dictionary strongly affects finite-dimensional Koopman approximation quality, motivating RBF dictionaries, data-driven centers, and dictionary-learning variants~\cite{li2017extended}.

\subsection{Lifted Control Barrier Functions}

We define a lifted Koopman control barrier function as
\begin{equation}
    h_K(z)=c^\top z+d
    \label{eq:linear_lifted_barrier}
\end{equation}
where $c\in\mathbb{R}^{n_z}$ and $d\in\mathbb{R}$. The corresponding safe set in physical coordinates is
\begin{equation}
    \mathcal{C}_K
    =
    \{y: h_K(\psi(y))\ge 0\}
    \label{eq:koopman_safe_set}
\end{equation}
When $z$ contains the original state coordinates, many physical constraints can be encoded directly. For example, if $p$ is a position coordinate included in $z$, then the upper-bound constraint $p\le p_{\max}$ can be written as
\begin{equation}
    h_K(z)=p_{\max}-p
\end{equation}

Using the nominal Koopman predictor, the one-step lifted CBF condition becomes
\begin{equation}
    h_K(Az_t+Bu_t)
    \ge
    (1-\eta)h_K(z_t)
    \label{eq:nominal_lifted_cbf}
\end{equation}
Since $h_K$ is affine and the lifted predictor is linear in $u_t$, the constraint is affine in the control input.

\subsection{Maximum-Entropy Actor--Critic Learning}

We use Soft Actor-Critic as the primary RL backbone because it is an off-policy maximum-entropy actor--critic method designed for continuous control~\cite{haarnoja2018soft}. SAC optimizes the maximum-entropy objective
\begin{equation}
    J(\pi)
    =
    \mathbb{E}_{\pi}
    \left[
    \sum_{t=0}^{\infty}
    \gamma^t
    \left(
    r(s_t,u_t)
    +
    \alpha \mathcal{H}(\pi(\cdot|s_t))
    \right)
    \right]
\end{equation}
where $\gamma\in(0,1)$ is the discount factor and $\alpha>0$ is the entropy temperature. In the proposed method, the actor produces a nominal action
\begin{equation}
    u_t^{\mathrm{nom}}\sim\pi_\theta(\cdot|s_t)
\end{equation}
which is passed through a Koopman-CBF safety filter before execution:
\begin{equation}
    u_t^{\mathrm{safe}}
    =
    \Pi_{\mathrm{KCBF}}(z_t,u_t^{\mathrm{nom}})
\end{equation}
The environment transition is therefore generated by $u_t^{\mathrm{safe}}$, not by $u_t^{\mathrm{nom}}$. This distinction is central to the critic update used in our method. PPO is a natural on-policy counterpart because it alternates data collection with multiple epochs of clipped surrogate policy optimization~\cite{schulman2017proximal}.


\subsection{Safety-Constrained Reinforcement Learning Objective}

We consider an infinite-horizon discounted Markov decision process
\begin{equation}
    \mathcal{M}
    =
    (\mathcal{S},\mathcal{A},P,r,\gamma)
\end{equation}
where $s_t\in\mathcal{S}$ is the RL observation, $u_t\in\mathcal{A}\equiv\mathcal{U}$ is the control action, $P$ is the transition kernel induced by the physical dynamics, $r$ is the reward, and $\gamma\in(0,1)$ is the discount factor. The physical state $x_t$ or modeling state $y_t$ may be partially embedded in $s_t$.

The goal is to learn a policy that maximizes expected return while maintaining safety:
\begin{align}
    \max_{\pi}
    \quad&
    \mathbb{E}_{\pi}
    \left[
        \sum_{t=0}^{\infty}\gamma^t r(s_t,u_t)
    \right]
    \label{eq:constrained_rl_objective}
    \\
    \text{s.t.}
    \quad&
    h_j(y_t)\ge 0,
    \qquad
    \forall t\ge 0,\quad j=1,\dots,J
    \label{eq:pointwise_safety_constraints}
\end{align}
Here $h_j$ denotes the $j$th safety constraint. Unlike constrained RL methods that enforce expected cumulative cost constraints, \eqref{eq:pointwise_safety_constraints} is a pointwise forward-invariance requirement. This formulation follows the safe RL objective of maximizing expected return while respecting safety constraints during learning or deployment~\cite{kushwaha2025review,brunke2022safe}. Unlike CMDP-based constrained RL, which typically enforces expected cumulative cost constraints~\cite{achiam2017constrained}, our formulation targets pointwise state safety through a per-step barrier condition.

\subsection{Available Data and Learned Dynamics Model}

The use of data-driven lifted predictors follows Koopman control methods in which nonlinear dynamics are embedded in a higher-dimensional observable space and then fit by least-squares regression~\cite{korda2018linear,williams2015data}. We assume access to an offline transition dataset for fitting the Koopman predictor:
\begin{equation}
    \mathcal{D}_K
    =
    \{(y_i,u_i,y_i^+)\}_{i=1}^{N}
\end{equation}
The dataset may be collected using random actions, a nominal controller, early-stage RL policies, or a mixture of controllers. From this dataset, we fit a finite-dimensional lifted linear model
\begin{equation}
    z^+ = A z + B u + r
    \qquad
    z=\psi(y)
\end{equation}
We further use a held-out calibration set
\begin{equation}
    \mathcal{D}_{\mathrm{cal}}
    =
    \{(y_i,u_i,y_i^+)\}_{i=1}^{N_{\mathrm{cal}}}
\end{equation}
to estimate the projected Koopman residuals associated with each lifted barrier:
\begin{equation}
    \delta_{i,j}
    =
    |c_j^\top r_i|
    =
    \left|
    c_j^\top
    \left[
    \psi(y_i^+) - A\psi(y_i)-Bu_i
    \right]
    \right|
\end{equation}
For each barrier $j$, we define a robust margin
\begin{equation}
    \rho_j
    =
    q_{1-\alpha}
    \left(
    \{\delta_{i,j}\}_{i=1}^{N_{\mathrm{cal}}}
    \right),
    \label{eq:rho_quantile}
\end{equation}
where $q_{1-\alpha}$ is a high empirical quantile, such as $0.95$ or $0.99$. A conformal variant can be obtained by using the split-conformal quantile index described in Theorem~\ref{thm:conformal_safety}.

\subsection{Robust Koopman-CBF Safety Requirement}

For each lifted barrier
\begin{equation}
    h_{K,j}(z)=c_j^\top z+d_j
\end{equation}
we impose the robust one-step condition
\begin{equation}
    h_{K,j}(Az_t+Bu_t)
    \ge
    (1-\eta_j)h_{K,j}(z_t)+\rho_j
    \qquad j=1,\dots,J
    \label{eq:robust_multicbf_condition}
\end{equation}
This constraint tightens the nominal lifted CBF condition by the projected residual margin $\rho_j$. The role of $\rho_j$ is to compensate for the fact that the true lifted successor is
\begin{equation}
    z_{t+1}
    =
    Az_t+Bu_t+r_t
\end{equation}

Substituting the affine form of $h_{K,j}$ into \eqref{eq:robust_multicbf_condition} yields
\begin{equation}
    c_j^\top B u_t
    \ge
    (1-\eta_j)(c_j^\top z_t+d_j)
    +
    \rho_j
    -
    c_j^\top A z_t
    -
    d_j
\end{equation}
Define
\begin{equation}
    a_j = B^\top c_j
\end{equation}
and
\begin{equation}
    b_j(z_t)
    =
    (1-\eta_j)(c_j^\top z_t+d_j)
    +
    \rho_j
    -
    c_j^\top A z_t
    -
    d_j
\end{equation}
Then the robust Koopman-CBF condition becomes the affine inequality
\begin{equation}
    a_j^\top u_t \ge b_j(z_t)
    \label{eq:affine_cbf_constraint}
\end{equation}

\subsection{Feasibility and Slack}

The robust CBF constraints may be infeasible under actuator limits or if the learned lifted model poorly represents the true dynamics. To prevent numerical failure, the implemented safety filter includes slack variables:
\begin{align}
    a_j^\top u_t+\xi_j &\ge b_j(z_t),
    \qquad j=1,\dots,J
    \\
    \xi_j &\ge 0
\end{align}
The slack variables are penalized heavily in the QP objective. Importantly, positive slack means that the strict CBF certificate is not enforced at that step. Therefore, all theoretical safety claims in this work are conditioned on either zero slack or on an explicitly bounded slack relaxation. In experiments, we report slack rate and slack magnitude separately from violation rate.

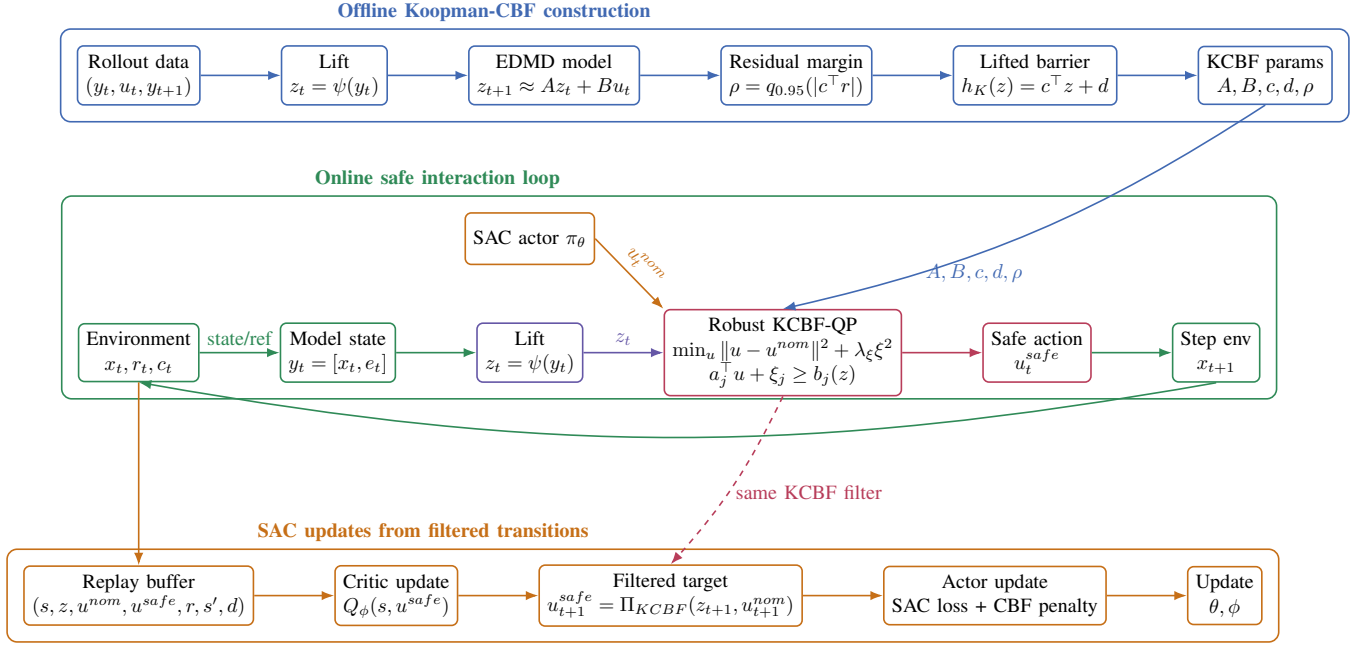
\begin{figure*}[t]
\centering
\resizebox{0.98\textwidth}{!}{%
\begin{tikzpicture}[
    font=\normalsize,
    box/.style={draw, rounded corners=3pt, thick, align=center,
                minimum height=9mm, inner sep=4pt},
    group/.style={draw, rounded corners=5pt, thick, inner sep=8pt},
    arr/.style={-{Latex[length=2.2mm]}, thick},
    darr/.style={-{Latex[length=2.2mm]}, thick, dashed}
]
\definecolor{off}{RGB}{65,105,178}
\definecolor{online}{RGB}{47,138,87}
\definecolor{learn}{RGB}{198,111,23}
\definecolor{filter}{RGB}{184,61,90}
\definecolor{model}{RGB}{101,87,166}

\node[box, draw=off] (data) {Rollout data\\$(y_t,u_t,y_{t+1})$};
\node[box, draw=off, right=14mm of data]  (lift0) {Lift\\$z_t=\psi(y_t)$};
\node[box, draw=off, right=14mm of lift0] (edmd)  {EDMD model\\$z_{t+1}\approx Az_t+Bu_t$};
\node[box, draw=off, right=14mm of edmd]  (rho)   {Residual margin\\$\rho=q_{0.95}(|c^\top r|)$};
\node[box, draw=off, right=14mm of rho]   (bar)   {Lifted barrier\\$h_K(z)=c^\top z+d$};
\node[box, draw=off, right=14mm of bar]   (saved) {KCBF params\\$A,B,c,d,\rho$};
\draw[arr, off] (data) -- (lift0);
\draw[arr, off] (lift0) -- (edmd);
\draw[arr, off] (edmd) -- (rho);
\draw[arr, off] (rho)  -- (bar);
\draw[arr, off] (bar)  -- (saved);

\node[box, draw=online, below=38mm of data] (env)    {Environment\\$x_t,r_t,c_t$};
\node[box, draw=online, right=14mm of env]  (ystate) {Model state\\$y_t=[x_t,e_t]$};
\node[box, draw=model,  right=14mm of ystate] (lift) {Lift\\$z_t=\psi(y_t)$};
\node[box, draw=learn,  above=10mm of lift] (actor)  {SAC actor $\pi_\theta$};
\node[box, draw=filter, right=14mm of lift, minimum width=35mm] (qp)
    {Robust KCBF-QP\\$\min_u\|u-u^{nom}\|^2+\lambda_\xi\xi^2$\\$a_j^\top u+\xi_j\geq b_j(z)$};
\node[box, draw=filter, right=14mm of qp]  (safe) {Safe action\\$u_t^{safe}$};
\node[box, draw=online, right=14mm of safe] (step) {Step env\\$x_{t+1}$};

\draw[arr, online] (env)    -- node[above]{state/ref} (ystate);
\draw[arr, online] (ystate) -- (lift);
\draw[arr, model]  (lift)   -- node[above]{$z_t$} (qp);
\draw[arr, learn]  (actor.east) -- node[above, pos=0.55, sloped]{$u_t^{nom}$} (qp.north west);
\draw[arr, filter] (qp)   -- (safe);
\draw[arr, online] (safe) -- (step);
\draw[arr, online] (step.south) to[bend left=10] (env.south);
\draw[arr, off]    (saved.south) to[bend left=10]
    node[right, pos=0.75]{$A,B,c,d,\rho$} (qp.north);

\node[box, draw=learn, below=32mm of env] (buffer)
    {Replay buffer\\$(s,z,u^{nom},u^{safe},r,s',d)$};
\node[box, draw=learn, right=14mm of buffer] (critic)
    {Critic update\\$Q_\phi(s,u^{safe})$};
\node[box, draw=learn, right=14mm of critic, minimum width=38mm] (target)
    {Filtered target\\$u_{t+1}^{safe}=\Pi_{KCBF}(z_{t+1},u_{t+1}^{nom})$};
\node[box, draw=learn, right=14mm of target] (actorupdate)
    {Actor update\\SAC loss + CBF penalty};
\node[box, draw=learn, right=14mm of actorupdate] (update)
    {Update\\$\theta,\phi$};

\draw[arr, learn] (env)         -- (buffer);
\draw[arr, learn] (buffer)      -- (critic);
\draw[arr, learn] (critic)      -- (target);
\draw[arr, learn] (target)      -- (actorupdate);
\draw[arr, learn] (actorupdate) -- (update);
\draw[darr, filter] (qp.south) to[bend left=8]
    node[right, pos=0.55]{same KCBF filter} (target.north);

\begin{scope}[on background layer]
\node[group, draw=off,    fit=(data)(saved),
      label={[off,    font=\bfseries]above left:Offline Koopman-CBF construction}] {};
\node[group, draw=online, fit=(env)(step)(actor),
      label={[online, font=\bfseries]above left:Online safe interaction loop}] {};
\node[group, draw=learn,  fit=(buffer)(update),
      label={[learn,  font=\bfseries]above left:SAC updates from filtered transitions}] {};
\end{scope}
\end{tikzpicture}%
}
\caption{Overall Robust Koopman-CBF SAC pipeline. Offline rollouts are used to fit a finite-dimensional Koopman predictor, construct a lifted barrier, and calibrate a projected residual margin. Online, the SAC actor proposes a nominal action that is projected through a robust Koopman-CBF QP before execution. The replay buffer stores both nominal and executed actions; critics are trained on executed safe actions, and the actor is regularized by a CBF feasibility penalty.}
\label{fig:kcbf_sac_pipeline}
\end{figure*}



\section{Methodology}
\label{sec:methodology}

\subsection{Overview}

The problem addressed in this work is, given sampled transitions from an unknown nonlinear robotic system, learn a Koopman predictor and a lifted CBF safety filter that can be integrated with SAC such that the executed policy improves task reward while satisfying pointwise safety constraints whenever the robust lifted CBF-QP is feasible and the projected Koopman residual is covered by the estimated margin. This problem contains three coupled challenges:
\begin{enumerate}
    \item learning a lifted dynamics model accurate enough for safety-critical prediction;
    \item constructing CBF constraints that are affine in the control input;
    \item integrating the resulting safety filter with actor--critic learning without creating a mismatch between nominal and executed actions.
\end{enumerate}

The proposed method, \textsc{Robust Koopman-CBF SAC}, consists of four stages:

\begin{enumerate}
    \item collect rollout data and fit a controlled Koopman predictor;
    \item construct lifted affine barrier functions and estimate projected residual margins;
    \item enforce robust Koopman-CBF constraints through a QP safety filter;
    \item train SAC using the executed safe actions and an actor-side CBF feasibility penalty.
\end{enumerate}

At deployment time, the actor proposes a nominal action $u_t^{\mathrm{nom}}$. The Koopman-CBF filter computes the executed action
\begin{equation}
    u_t^{\mathrm{safe}}
    =
    \Pi_{\mathrm{KCBF}}(z_t,u_t^{\mathrm{nom}})
\end{equation}
and the environment evolves according to
\begin{equation}
    x_{t+1}=f(x_t,u_t^{\mathrm{safe}})
\end{equation}

\subsection{Koopman Model Learning}

Given $\mathcal{D}_K=\{(y_i,u_i,y_i^+)\}_{i=1}^{N}$, we compute lifted states
\begin{equation}
    z_i=\psi(y_i),
    \qquad
    z_i^+=\psi(y_i^+)
\end{equation}
We fit the linear lifted model
\begin{equation}
    z_i^+\approx A z_i+B u_i
\end{equation}
using ridge regression as in \eqref{eq:edmd_regression}. This least-squares fitting procedure is the standard finite-dimensional EDMD/DMDc approach for learning lifted linear predictors from trajectory data~\cite{williams2015data,proctor2016dynamic,korda2018linear}. The resulting model is validated using both one-step and multi-step prediction error:
\begin{equation}
    \mathrm{MSE}_1
    =
    \frac{1}{N}
    \sum_{i=1}^{N}
    \|z_i^+ - Az_i-Bu_i\|_2^2
\end{equation}
and
\begin{equation}
    \mathrm{MSE}_H
    =
    \frac{1}{N}
    \sum_{i=1}^{N}
    \sum_{k=1}^{H}
    \|z_{i+k}-\hat z_{i+k}\|_2^2
\end{equation}
where
\begin{equation}
    \hat z_{i+k+1}
    =
    A\hat z_{i+k}+B u_{i+k}
\end{equation}

\subsection{Lifted Barrier Construction}

For each safety constraint, we construct a lifted affine barrier
\begin{equation}
    h_{K,j}(z)=c_j^\top z+d_j
\end{equation}

A key advantage of Koopman lifting for control is that nonlinear state functions and constraints can sometimes be represented linearly in the lifted coordinates, enabling linear or convex online control formulations~\cite{korda2018linear}. When a safety constraint is linear in the original state and the original coordinates are included in $z$, the lifted barrier can be constructed directly. For example, the cart-position constraint $p\le p_{\max}$ becomes
\begin{equation}
    h_K(z)=p_{\max}-p
\end{equation}
If both upper and lower bounds are required, we use two barriers:
\begin{equation}
    h_{+}(z)=p_{\max}-p,
    \qquad
    h_{-}(z)=p-p_{\min}
\end{equation}

For velocity-constrained locomotion tasks, one may either include nonlinear velocity features such as $v^2$ in $\psi$, or fit a linear separator in lifted space from safe/unsafe labels. The latter produces
\begin{equation}
    h_K(z)=c^\top z+d
\end{equation}
by supervised classification. However, such learned barriers must be validated carefully because classification accuracy alone does not imply forward invariance.

\subsection{Projected Residual Margin}

The use of a residual margin is motivated by robust CBF theory, which shows that barrier conditions must account for model uncertainty or disturbances to preserve invariance~\cite{xu2015robustness}. For each barrier $h_{K,j}$, the Koopman residual affects safety through $c_j^\top r_t$. Therefore, instead of bounding the full residual norm, we estimate the projected residual:
\begin{equation}
    \delta_{i,j}
    =
    \left|
    c_j^\top
    \left(
    z_i^+ - A z_i - B u_i
    \right)
    \right|
\end{equation}
The robust margin is chosen as
\begin{equation}
    \rho_j
    =
    q_{1-\alpha}
    \left(
    \{\delta_{i,j}\}_{i=1}^{N_{\mathrm{cal}}}
    \right)
\end{equation}
This produces less conservative constraints than using a full-norm residual bound:
\begin{equation}
    |c_j^\top r_t|
    \le
    \|c_j\|_2\|r_t\|_2
\end{equation}

When the margin is chosen using split conformal quantiles, the resulting residual coverage follows the exchangeability-based finite-sample guarantees of conformal prediction~\cite{vovk2005algorithmic,angelopoulos2023conformal}.
\subsection{Robust Koopman-CBF Safety Filter}

Our safety layer builds on Koopman-CBF synthesis, where the nonlinear system is represented in lifted coordinates and CBF constraints are imposed in the lifted space~\cite{folkestad2020data,zinage2023neural}. For a nominal actor action $u_t^{\mathrm{nom}}$, the safety filter solves
\begin{align}
    u_t^{\mathrm{safe}},\xi_t
    =
    \arg\min_{u,\xi}
    \quad&
    \frac{1}{2}\|u-u_t^{\mathrm{nom}}\|_2^2
    +
    \lambda_\xi
    \sum_{j=1}^{J}
    \xi_j^2
    \label{eq:kcbf_qp}
    \\
    \text{s.t.}
    \quad&
    a_j^\top u+\xi_j
    \ge
    b_j(z_t),
    \quad j=1,\dots,J
    \nonumber
    \\
    &
    u_{\min}\le u\le u_{\max},
    \nonumber
    \\
    &
    \xi_j\ge 0,
    \qquad j=1,\dots,J
    \nonumber
\end{align}
When the QP is feasible with $\xi_j=0$ for all $j$, the robust CBF constraints are exactly enforced. When slack is positive, the filter remains numerically well-defined but the formal certificate is weakened. The filter is a convex quadratic program because the objective is quadratic and the robust Koopman-CBF and box-input constraints are affine~\cite{boyd2004convex}. Algorithm \ref{alg:kcbf_sac} denotes the overall KCBF-SAC algorithm.

\subsection{Integration with SAC}

The replay buffer stores both nominal and executed actions:
\begin{equation}
    \mathcal{B}
    =
    \{(s_t,z_t,u_t^{\mathrm{nom}},u_t^{\mathrm{safe}},r_t,s_{t+1},z_{t+1},d_t)\}
\end{equation}
The SAC updates follow the maximum-entropy off-policy actor--critic framework~\cite{haarnoja2018soft}, but the critic is trained on the filtered action actually executed by the environment. The critic is trained using the executed action:
\begin{equation}
    \mathcal{L}_Q(\phi_i)
    =
    \mathbb{E}_{\mathcal{B}}
    \left[
    \left(
    Q_{\phi_i}(s_t,u_t^{\mathrm{safe}})
    -
    y_t
    \right)^2
    \right]
\end{equation}
For the target, the next nominal action is sampled from the current actor,
\begin{equation}
    u_{t+1}^{\mathrm{nom}}\sim\pi_\theta(\cdot|s_{t+1})
\end{equation}
then filtered:
\begin{equation}
    u_{t+1}^{\mathrm{safe}}
    =
    \Pi_{\mathrm{KCBF}}(z_{t+1},u_{t+1}^{\mathrm{nom}})
\end{equation}
The soft target is
\begin{equation}
\begin{aligned}
    y_t
    = &  r_t     +     \gamma(1-d_t)
    [
    \min_{i=1,2}
    Q_{\bar{\phi}_i}(s_{t+1},u_{t+1}^{\mathrm{safe}}) \\
    &-
    \alpha
    \log\pi_\theta(u_{t+1}^{\mathrm{nom}}|s_{t+1})
    ]
    \label{eq:kcbf_sac_target}
\end{aligned}
\end{equation}

CBF-based safety layers have previously been used to project RL actions onto safe action sets and to guide exploration during learning~\cite{cheng2019end,emam2022safe}. The actor loss includes the standard SAC term and a CBF proposal penalty:
\begin{equation}
\begin{aligned}
    \mathcal{L}_{\pi}(\theta)
    =&
    \mathbb{E}_{s_t\sim\mathcal{B},\,u_t^{\mathrm{nom}}\sim\pi_\theta}
    [
    \alpha\log\pi_\theta(u_t^{\mathrm{nom}}|s_t) \\ 
    &- \min_i Q_{\phi_i}(s_t,u_t^{\mathrm{nom}})
    +
    \lambda_h
    \ell_{\mathrm{cbf}}(z_t,u_t^{\mathrm{nom}})
    ]
    \label{eq:actor_loss}
    \end{aligned}
\end{equation}
where
\begin{equation}
    \ell_{\mathrm{cbf}}(z,u)
    =
    \sum_{j=1}^{J}
    \left[
    \max
    \left(
    0,
    b_j(z)-a_j^\top u
    \right)
    \right]^2
    \label{eq:cbf_loss}
\end{equation}
The penalty encourages the actor to produce actions that are already feasible, thereby reducing filter intervention over training.

\begin{algorithm}[t]
\caption{\textsc{Robust Koopman-CBF SAC}}
\label{alg:kcbf_sac}
\begin{algorithmic}[1]
\STATE Collect Koopman dataset $\mathcal{D}_K=\{(y_i,u_i,y_i^+)\}_{i=1}^{N}$.
\STATE Fit lifting map $\psi$ and EDMD model $z^+=Az+Bu$.
\STATE Construct lifted barriers $h_{K,j}(z)=c_j^\top z+d_j$.
\STATE Estimate projected residual margins $\rho_j$ from held-out calibration data.
\STATE Initialize SAC actor $\pi_\theta$, critics $Q_{\phi_1},Q_{\phi_2}$, target critics, and replay buffer $\mathcal{B}$.
\FOR{environment step $t=0,1,2,\dots$}
    \STATE Observe $s_t$ and construct modeling state $y_t$.
    \STATE Lift $z_t=\psi(y_t)$.
    \STATE Sample nominal action $u_t^{\mathrm{nom}}\sim\pi_\theta(\cdot|s_t)$.
    \STATE Solve the robust Koopman-CBF QP \eqref{eq:kcbf_qp} to obtain $u_t^{\mathrm{safe}}$.
    \STATE Execute $u_t^{\mathrm{safe}}$ in the environment.
    \STATE Store $(s_t,z_t,u_t^{\mathrm{nom}},u_t^{\mathrm{safe}},r_t,s_{t+1},z_{t+1},d_t)$ in $\mathcal{B}$.
    \IF{update step}
        \STATE Sample minibatch from $\mathcal{B}$.
        \STATE Compute filtered next actions using $\Pi_{\mathrm{KCBF}}$.
        \STATE Update critics using \eqref{eq:kcbf_sac_target}.
        \STATE Update actor using \eqref{eq:actor_loss}.
        \STATE Update entropy temperature and target critics.
    \ENDIF
\ENDFOR
\end{algorithmic}
\end{algorithm}

\section{Theoretical Analysis}
\label{sec:theory}

We now state the main safety properties of the robust Koopman-CBF filter. The results are intentionally conditional: the safety guarantee requires that the QP be feasible without slack and that the projected Koopman residual be covered by the chosen margin. These conditions are measurable in experiments through slack rate, CBF gap, and residual calibration.

\subsection{Nominal Discrete-Time CBF Invariance}

The proof follows the standard discrete-time CBF invariance argument: if $h(x_{t+1})\ge (1-\eta)h(x_t)$ and $h(x_t)\ge 0$, then $h(x_{t+1})\ge 0$~\cite{agrawal2017discrete}.

\begin{assumption} For each physical constraint set \(\mathcal{C}_j^{\rm phys}\), the lifted barrier safe set is an inner approximation:
	\[
	\{y:h_{K,j}(\psi(y))\ge 0\}\subseteq \mathcal{C}_j^{\rm phys}.
	\]
\end{assumption}

\begin{theorem}[Discrete-Time CBF Forward Invariance]
\label{thm:dt_cbf_invariance}
Consider the system $x_{t+1}=f(x_t,u_t)$ and safe set $\mathcal{C}=\{x:h(x)\ge 0\}$. Suppose that for every $x_t\in\mathcal{C}$, the applied control $u_t$ satisfies
\begin{equation}
    h(f(x_t,u_t))
    \ge
    (1-\eta)h(x_t),
    \qquad
    \eta\in(0,1]
\end{equation}
If $x_0\in\mathcal{C}$, then $x_t\in\mathcal{C}$ for all $t\ge 0$.
\end{theorem}

\begin{proof}
Since $x_0\in\mathcal{C}$, we have $h(x_0)\ge 0$. Applying the CBF condition at $t=0$ gives
\begin{equation}
    h(x_1)
    \ge
    (1-\eta)h(x_0)
    \ge
    0
\end{equation}
because $1-\eta\ge 0$. Hence $x_1\in\mathcal{C}$. Repeating the same argument inductively, if $h(x_t)\ge 0$, then
\begin{equation}
    h(x_{t+1})
    \ge
    (1-\eta)h(x_t)
    \ge
    0
\end{equation}
Therefore $x_t\in\mathcal{C}$ for all $t\ge 0$.
\end{proof}

\subsection{Convexity of the Robust Koopman-CBF Filter}

The convexity result follows immediately from standard convex optimization theory for quadratic objectives with affine constraints~\cite{boyd2004convex}.

\begin{proposition}[Convexity of the KCBF-QP]
\label{prop:qp_convexity}
For affine lifted barriers $h_{K,j}(z)=c_j^\top z+d_j$ and a linear lifted predictor $z^+=Az+Bu$, the robust Koopman-CBF filter \eqref{eq:kcbf_qp} is a convex quadratic program. If $\lambda_\xi>0$ and the feasible set is nonempty, the optimal pair $(u^{\mathrm{safe}},\xi)$ is unique in the objective variables.
\end{proposition}

\begin{proof}
The objective
\begin{equation}
    \frac{1}{2}\|u-u^{\mathrm{nom}}\|_2^2
    +
    \lambda_\xi
    \sum_{j=1}^{J}\xi_j^2
\end{equation}
is a strictly convex quadratic function of $(u,\xi)$ when $\lambda_\xi>0$. Each robust CBF constraint has the form
\begin{equation}
    a_j^\top u+\xi_j\ge b_j(z)
\end{equation}
which is affine in $(u,\xi)$. The input bounds and nonnegativity constraints on $\xi$ are also affine. Therefore the feasible set is convex, and the optimization problem is a convex QP. Strict convexity of the objective implies uniqueness of the minimizer whenever the feasible set is nonempty.
\end{proof}

\subsection{Robust Koopman-CBF Invariance}

The robust version parallels robust CBF arguments in which model error enters as an additive disturbance that must be bounded or compensated in the barrier condition~\cite{xu2015robustness}. 

\begin{assumption}[Lifted Residual Coverage]
\label{assump:residual_coverage}
For each barrier $j$, the projected Koopman residual satisfies
\begin{equation}
    |c_j^\top r_t|
    \le
    \rho_j
\end{equation}
for all time steps under consideration.
\end{assumption}
The residual model captures finite-dimensional Koopman approximation error, a known issue in EDMD-style approximations of nonlinear dynamics~\cite{li2017extended}. 
\begin{assumption}[Zero-Slack Feasibility]
\label{assump:zero_slack}
At each time step, the KCBF-QP admits a feasible solution with $\xi_j=0$ for all $j$.
\end{assumption}

\begin{theorem}[Robust Koopman-CBF Forward Invariance]
\label{thm:robust_kcbf_invariance}
Consider the lifted dynamics
\begin{equation}
    z_{t+1}=Az_t+Bu_t+r_t
\end{equation}
and lifted safe set
\begin{equation}
    \mathcal{C}_K
    =
    \{z:h_{K,j}(z)\ge 0,\ j=1,\dots,J\}
\end{equation}
Suppose Assumptions~\ref{assump:residual_coverage} and~\ref{assump:zero_slack} hold. If the applied control satisfies
\begin{equation}
    h_{K,j}(Az_t+Bu_t)
    \ge
    (1-\eta_j)h_{K,j}(z_t)+\rho_j,
    \quad
    j=1,\dots,J
    \label{eq:robust_condition_theorem}
\end{equation}
then $\mathcal{C}_K$ is forward invariant. That is, if $z_0\in\mathcal{C}_K$, then $z_t\in\mathcal{C}_K$ for all $t\ge 0$.
\end{theorem}

\begin{proof}
For each barrier $j$, the true next lifted state is
\begin{equation}
    z_{t+1}=Az_t+Bu_t+r_t
\end{equation}
Using the affine form of the barrier,
\begin{align}
    h_{K,j}(z_{t+1})
    &=
    h_{K,j}(Az_t+Bu_t+r_t)
    \\
    &=
    c_j^\top(Az_t+Bu_t+r_t)+d_j
    \\
    &=
    h_{K,j}(Az_t+Bu_t)+c_j^\top r_t
\end{align}
By Assumption~\ref{assump:residual_coverage},
\begin{equation}
    c_j^\top r_t
    \ge
    -|c_j^\top r_t|
    \ge
    -\rho_j
\end{equation}
Therefore,
\begin{align}
    h_{K,j}(z_{t+1})
    &\ge
    h_{K,j}(Az_t+Bu_t)-\rho_j
\end{align}
Using the robust CBF condition \eqref{eq:robust_condition_theorem},
\begin{align}
    h_{K,j}(z_{t+1})
    &\ge
    (1-\eta_j)h_{K,j}(z_t)+\rho_j-\rho_j
    \\
    &=
    (1-\eta_j)h_{K,j}(z_t)
\end{align}
If $h_{K,j}(z_t)\ge 0$ and $\eta_j\in(0,1]$, then
\begin{equation}
    h_{K,j}(z_{t+1})\ge 0
\end{equation}
Since this holds for every $j=1,\dots,J$, $z_t\in\mathcal{C}_K$ implies $z_{t+1}\in\mathcal{C}_K$. By induction, $\mathcal{C}_K$ is forward invariant.
\end{proof}

\begin{remark}[Effect of Slack]
If the QP is solved with nonzero slack, the enforced constraint becomes
\begin{equation}
    h_{K,j}(Az_t+Bu_t)
    \ge
    (1-\eta_j)h_{K,j}(z_t)+\rho_j-\xi_j
\end{equation}
The same proof yields
\begin{equation}
    h_{K,j}(z_{t+1})
    \ge
    (1-\eta_j)h_{K,j}(z_t)-\xi_j
\end{equation}
Thus positive slack weakens the invariance guarantee. This is why slack rate and slack magnitude must be reported separately in experiments.
\end{remark}

\subsection{High-Probability Safety under Conformal Residual Calibration}

The previous theorem is deterministic conditional on residual coverage. We now state a high-probability version using split conformal calibration.

\begin{assumption}[Exchangeable Calibration Residuals]
\label{assump:exchangeability}
For a fixed barrier $j$, the calibration residuals
\begin{equation}
\delta_{1,j}, \ldots, \delta_{N_{\text{cal}},j}
\end{equation}
and a future deployment residual $\delta_{*,j} = |c_j^\top r_*|$ are
exchangeable. We emphasize that this exchangeability is with respect to
the joint distribution induced by the data-collection policy used to
generate $\mathcal{D}_{\text{cal}}$, not the deployed SAC policy.
\end{assumption}

Let
\begin{equation}
    k
    =
    \left\lceil
    (N_{\mathrm{cal}}+1)(1-\alpha)
    \right\rceil
\end{equation}
and define $\rho_j$ as the $k$th order statistic of the calibration residuals:
\begin{equation}
    \rho_j
    =
    \delta_{(k),j}
\end{equation}
If $k>N_{\mathrm{cal}}$, set $\rho_j=+\infty$.

\begin{theorem}[One-Step High-Probability Residual Coverage]
\label{thm:conformal_safety}
Under Assumption~\ref{assump:exchangeability},
\begin{equation}
    \mathbb{P}
    \left(
    |c_j^\top r_*|\le \rho_j
    \right)
    \ge
    1-\alpha
\end{equation}
Consequently, if the zero-slack robust KCBF condition is enforced using $\rho_j$, then the one-step safety implication in Theorem~\ref{thm:robust_kcbf_invariance} holds with probability at least $1-\alpha$ for barrier $j$.
\end{theorem}

\begin{proof}
The finite-sample residual coverage statement follows from split conformal prediction: under exchangeability, an empirical calibration quantile gives marginal coverage without distributional assumptions~\cite{vovk2005algorithmic,angelopoulos2023conformal}. By exchangeability, the rank of the future residual $\delta_{*,j}$ among the $N_{\mathrm{cal}}+1$ residuals is uniformly distributed. Choosing the $k$th order statistic with
\begin{equation}
    k=\lceil(N_{\mathrm{cal}}+1)(1-\alpha)\rceil
\end{equation}
ensures that the probability of the future residual exceeding $\rho_j$ is at most $\alpha$. Therefore,
\begin{equation}
    \mathbb{P}(\delta_{*,j}\le \rho_j)\ge 1-\alpha
\end{equation}
Conditioned on the event $\delta_{*,j}\le\rho_j$, the proof of Theorem~\ref{thm:robust_kcbf_invariance} applies for one step.
\end{proof}

\begin{remark}[Coverage under policy-induced distribution shift]
\label{rem:shift}
The coverage guarantee in Theorem~\ref{thm:conformal_safety} is marginal over
the calibration distribution. In the safety-filtered RL setting,
$\mathcal{D}_{\text{cal}}$ is collected once from random or nominal
rollouts, while the deployed actor $\pi_\theta$ evolves throughout
training and induces a different state-action visitation distribution.
The exchangeability of $\delta_{*,j}$ with the calibration residuals
therefore holds only when the deployed policy's visitation distribution
matches that of the calibration policy, i.e., a condition that is
generally violated during learning. In practice, the projected
residual margin $\rho_j$ should be treated as an empirically reliable
but distributionally biased estimate of the deployment residual quantile;
its validity should be monitored online by tracking the realized
projected residuals $|c_j^\top r_t|$ against $\rho_j$ during training.
Distribution-shift-aware conformal variants, including weighted
conformal prediction~\cite{tibshirani2019conformal} and adaptive
online conformal inference~\cite{gibbs2021adaptive}, provide principled
routes to tighter coverage in this regime and are an important
direction for future work.
\end{remark}

\begin{corollary}[Finite-Horizon Safety by Union Bound]
\label{cor:finite_horizon}
Suppose the conformal residual coverage condition in Theorem~\ref{thm:conformal_safety} holds at each time step for each barrier with failure probability $\alpha'$. Over a horizon $T$ and $J$ barriers, the probability that all residuals are covered is at least
\begin{equation}
    1-TJ\alpha'
\end{equation}
Therefore, choosing $\alpha'=\delta/(TJ)$ yields finite-horizon residual coverage with probability at least $1-\delta$.
\end{corollary}

\begin{proof}
Let $E_{t,j}$ denote the event that the residual for barrier $j$ at time $t$ is not covered by its margin. If
\begin{equation}
    \mathbb{P}(E_{t,j})\le \alpha'
\end{equation}
then by the union bound,
\begin{equation}
    \mathbb{P}
    \left(
    \bigcup_{t=0}^{T-1}
    \bigcup_{j=1}^{J}
    E_{t,j}
    \right)
    \le
    TJ\alpha'
\end{equation}
Thus all residuals are covered with probability at least $1-TJ\alpha'$. Setting $\alpha'=\delta/(TJ)$ gives the result.
\end{proof}

\begin{remark}[Practical tightness of the union bound]
\label{rem:union-tight}
The bound $1 - TJ\alpha'$ in Corollary~\ref{cor:finite_horizon} is correct but
loose for typical RL horizons. Concretely, with $T = 1000$, $J = 2$,
and target $\delta = 0.02$, the required per-step coverage is
$\alpha' = \delta/(TJ) = 10^{-5}$. The corresponding split-conformal
quantile index is
\begin{equation}
k = \lceil (N_{\text{cal}}+1)(1 - \alpha') \rceil
   = \lceil 2001 \times (1 - 10^{-5}) \rceil = 2001
\end{equation}
which exceeds $N_{\text{cal}} = 2000$, forcing $\rho_j = +\infty$ and
rendering the robust KCBF constraint trivially satisfied. More
generally, with $N_{\text{cal}} = 2000$ the smallest resolvable
per-step level is $\alpha' \approx 1/(N_{\text{cal}}+1) \approx
5 \times 10^{-4}$, so the union-bound guarantee becomes vacuous
($TJ\alpha' \geq 1$) whenever $TJ \geq 2000$. Achieving non-trivial
finite-horizon coverage at episode lengths typical of locomotion tasks
therefore requires either substantially larger calibration sets
(e.g., $N_{\text{cal}} \gtrsim 10^5$ for $T = 1000$), tighter
horizon-dependent concentration arguments, or online conformal
recalibration. The empirical residuals reported in
Section~\ref{sec:experiments} should therefore be interpreted as a marginal
calibration diagnostic, not as a certified finite-horizon coverage
statement.
\end{remark}

\subsection{Consistency of Filtered Critic Targets}

The safety filter changes the action executed by the environment. Therefore, critic training should evaluate the filtered closed-loop policy rather than the raw actor policy.

Define the filtered policy
\begin{equation}
    \pi_{\theta}^{\mathrm{safe}}(u^{\mathrm{safe}}|s)
\end{equation}
as the distribution induced by sampling
\begin{equation}
    u^{\mathrm{nom}}\sim\pi_\theta(\cdot|s)
\end{equation}
and applying
\begin{equation}
    u^{\mathrm{safe}}=\Pi_{\mathrm{KCBF}}(z,u^{\mathrm{nom}})
\end{equation}

\begin{proposition}[Transition Consistency of Executed-Action Critics]
\label{prop:filtered_critic_consistency}
Suppose environment transitions are generated by $u_t^{\mathrm{safe}}=\Pi_{\mathrm{KCBF}}(z_t,u_t^{\mathrm{nom}})$. Then a critic trained on tuples $(s_t,u_t^{\mathrm{safe}},r_t,s_{t+1})$ estimates the action-value function of the filtered closed-loop policy. In contrast, a critic trained only on $u_t^{\mathrm{nom}}$ is generally off-transition unless $u_t^{\mathrm{nom}}=u_t^{\mathrm{safe}}$ almost surely.
\end{proposition}

\begin{proof}
The action-value function for the executed closed-loop policy is
\begin{equation}
    Q^{\pi^{\mathrm{safe}}}(s_t,u_t^{\mathrm{safe}})
    =
    \mathbb{E}
    \left[
    r_t
    +
    \gamma
    V^{\pi^{\mathrm{safe}}}(s_{t+1})
    \mid
    s_t,u_t^{\mathrm{safe}}
    \right]
\end{equation}
where the transition distribution is induced by the physical environment under the executed action $u_t^{\mathrm{safe}}$. Since the replay tuples are generated by applying $u_t^{\mathrm{safe}}$ to the environment, the tuple $(s_t,u_t^{\mathrm{safe}},r_t,s_{t+1})$ is on-transition for this Bellman equation. If instead the critic is trained on $(s_t,u_t^{\mathrm{nom}},r_t,s_{t+1})$ while $u_t^{\mathrm{nom}}\neq u_t^{\mathrm{safe}}$, then the recorded successor state $s_{t+1}$ is not generated by $u_t^{\mathrm{nom}}$. Therefore the tuple is generally inconsistent with the Bellman equation for $Q(s,u_t^{\mathrm{nom}})$.
\end{proof}

\subsection{Safe Policy Improvement over the Feasible Action Set}

For a fixed lifted model and fixed KCBF constraints, define the feasible action set
\begin{equation}
    \mathcal{U}_{K}(z)
    =
    \{u\in\mathcal{U}: a_j^\top u\ge b_j(z),\ j=1,\dots,J\}
\end{equation}
When this set is nonempty and residual coverage holds, actions in $\mathcal{U}_K(z)$ satisfy the robust KCBF condition.

\begin{remark}[Actor regularization toward the feasible set]
\label{rem:actor-reg}
For a fixed lifted model and fixed KCBF constraints, define the
feasible action set
\begin{equation}
\mathcal{U}_K(z) = \{u \in \mathcal{U} :
  a_j^\top u \geq b_j(z),\; j = 1,\ldots,J\}.
\end{equation}
Actions in $\mathcal{U}_K(z)$ satisfy the robust KCBF condition under
residual coverage. The actor-side penalty $\ell_{\text{cbf}}$ in
(\ref{eq:cbf_loss}) is therefore a soft regularizer toward this set:
it does not enforce $\mathrm{supp}(\pi_\theta(\cdot|s)) \subseteq
\mathcal{U}_K(z)$, but reduces the expected CBF gap and the empirical
intervention rate over training. A formal policy-improvement
guarantee restricted to $\mathcal{U}_K(z)$ is not claimed here, since
the SAC soft-greedy update with a quadratic penalty does not reduce
to the constrained Bellman operator; we instead report intervention
rate and slack rate as empirical diagnostics of feasibility
(Section~\ref{sec:experiments}).
\end{remark}

\begin{figure*}[!t]
    \centering
    \includegraphics[width=\textwidth]{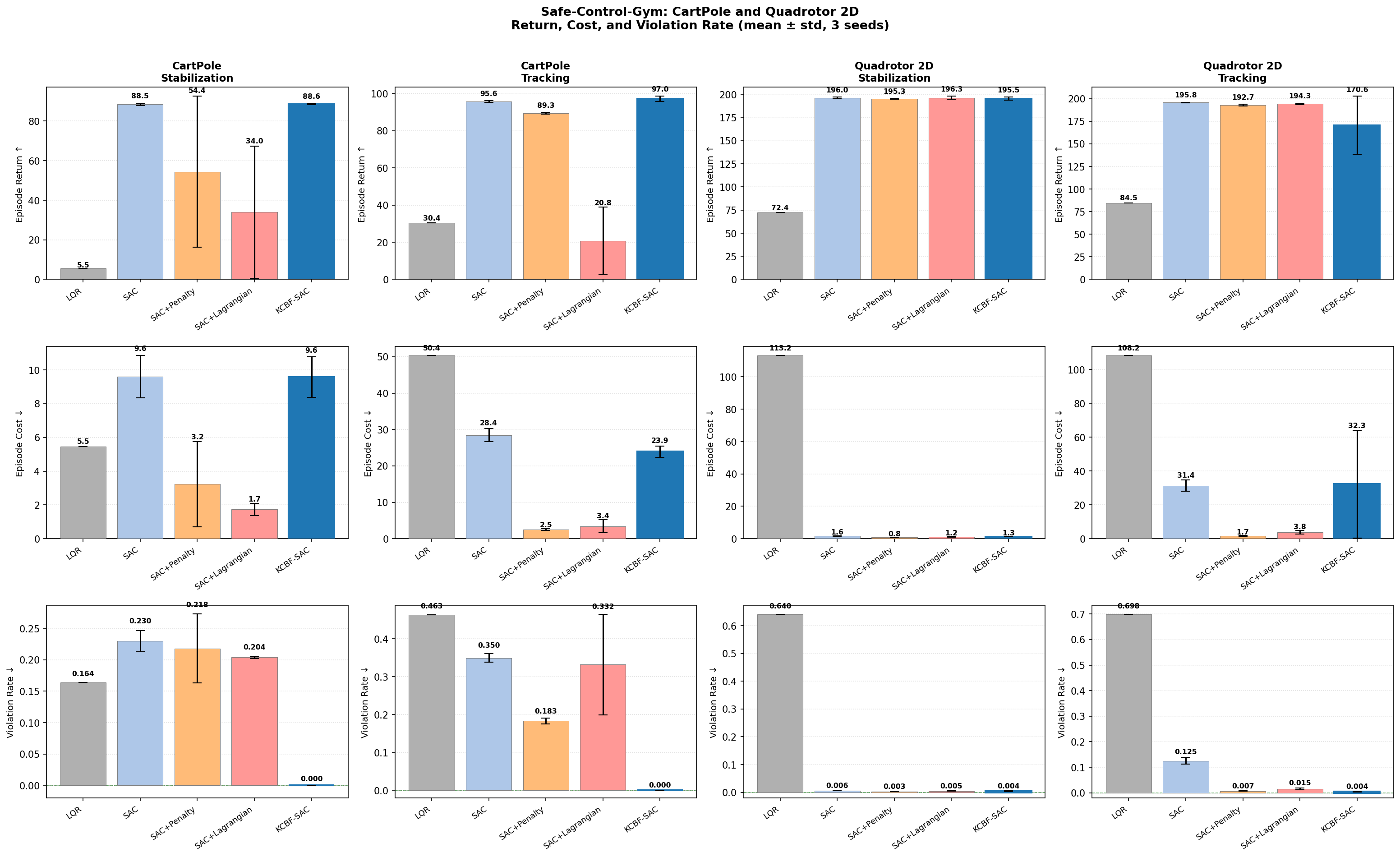}
    \caption{Safe-Control-Gym results: episode return, cost, and violation rate for CartPole (stabilization and tracking) and Quadrotor 2D (stabilization and tracking). KCBF-SAC (dark blue, thick border) achieves zero violations on CartPole and $<0.5\%$ on Quadrotor while matching SAC return. Error bars are mean\,$\pm$\,std over 3 seeds.}
    \label{fig:scg_comparison}
\end{figure*}


\section{Experiments}
\label{sec:experiments}

\subsection{Benchmark Environments}

We evaluate on two complementary benchmark families.
\textit{Safe-Control-Gym}~\cite{yuan2022safe} provides low-dimensional robotic environments with symbolic dynamics, explicit state constraints, and disturbance injection.
We use four tasks: CartPole stabilization and reference tracking, and Quadrotor 2D stabilization and trajectory tracking.
For CartPole, the state is $(p,\theta,\dot{p},\dot{\theta})$ with continuous force input, and the safety constraint is the cart-position bound $p\in[-0.2,0.2]$\,m.
For Quadrotor 2D, the state is $(x,y,\phi,\dot{x},\dot{y},\dot{\phi})$ with two-dimensional thrust input, and the safety constraint is a minimum altitude $y\geq y_{\min}$.

\textit{Safety Gymnasium}~\cite{ji2023safety} provides MuJoCo-based locomotion tasks with safety costs.
We use \texttt{SafetyWalker2d-v1} (17-dim state, 6-dim action) and \texttt{SafetyHalfCheetah-v1} (17-dim state, 6-dim action), both with a forward-velocity safety constraint $v_x\leq v_{\max}$.
These tasks expose the method to contact-rich dynamics where linear Koopman approximations are significantly more challenging.

\subsection{Baselines}

We compare against four baselines:
\textit{LQR}~(Safe-Control-Gym only): a model-based linear controller using symbolic dynamics, providing a reference point without learning;
\textit{SAC}~\cite{haarnoja2018soft}: unconstrained Soft Actor-Critic;
\textit{SAC+Penalty} \cite{lu2022reward} — SAC with a reward penalty for constraint violations (weight $c_{\mathrm{pen}}=1$);
\textit{SAC+Lagrangian}~\cite{achiam2017constrained, kushwaha2026lyapunov}: SAC with an adaptive dual variable enforcing a cumulative cost constraint.

\subsection{Implementation Details}

\subsubsection{Koopman model}
For each environment we collect $N=10{,}000$ transitions from random rollouts.
The observable dictionary consists of the original state coordinates augmented with $M=32$ RBF features with centers placed by $k$-means clustering of the collected data.
The controlled lifted model $z^+=Az+Bu$ is fit by ridge regression with $\lambda_{\mathrm{ridge}}=10^{-4}$.
The projected residual margin $\rho$ is estimated as the 95th percentile of $\{|c_j^\top r_i|\}$ on a held-out calibration set of 2{,}000 transitions, as defined in Section~\ref{sec:theory}.

\subsubsection{Barrier functions}
Barrier design is environment-specific and critical for filter effectiveness.
For CartPole, we use two affine barriers: $h_+(z)=p_{\max}-p$ and $h_-(z)=p-p_{\min}$, which are exactly linear in $z$ since $p$ is included in $\psi$.
For Quadrotor 2D, the altitude constraint $y\geq y_{\min}$ has relative degree 2 with respect to vertical thrust ($c^\top B=0$ for the naïve barrier), making the standard CBF-QP degenerate.
We resolve this by using a composite barrier
\begin{equation}
    h(z) = \alpha(y - y_{\min}) + \beta\,\dot{y}, \quad \alpha,\beta>0,
    \label{eq:composite_barrier}
\end{equation}
which achieves relative degree~1 while encoding both altitude and vertical velocity in the safety margin~\cite{ames2019control}.
For Safety Gymnasium, we use a velocity barrier $h_K(z)=v_{\max}-v_x$, where $v_x$ is included in the original coordinates.

\subsubsection{KCBF-SAC hyperparameters}
The CBF discount is $\eta=0.9$ and the slack weight is $\lambda_\xi=10^4$ for all environments.
The actor CBF penalty weight is $\lambda_h=1.0$.
SAC uses discount $\gamma=0.99$, Polyak parameter $\tau=0.005$, batch size $256$, and learning rate $3\times10^{-4}$.
Training budgets are $150{,}000$ steps for CartPole, $500{,}000$ steps for Quadrotor, and $1{,}000{,}000$ steps for Safety Gymnasium tasks.

\begin{remark}
    Hyperparameter selection is based on \textit{Safe-control-gym} and \textit{Safety-gym} tuned hyperparameters for SAC. 
\end{remark}

\subsection{Evaluation}

All configurations use 3 independent random seeds.
We evaluate for 10 episodes every 5{,}000 environment steps and report means at the final checkpoint.
Primary metrics are: \textit{Episode Return}~($\uparrow$), \textit{Violation Rate} (fraction of steps with $h(x)<0$, $\downarrow$), \textit{Intervention Rate} (fraction of steps where $\|u^{\mathrm{safe}}-u^{\mathrm{nom}}\|_2>\varepsilon$), and \textit{Min-h} (minimum barrier value per episode).
We also report \textit{Slack Rate} (fraction of steps with $\xi>0$) as a direct feasibility diagnostic.


\section{Results, Discussion, and Limitations}
\label{sec:results}

Table~\ref{tab:results} summarizes evaluation performance across all environments and methods. Figure~\ref{fig:scg_comparison} and \ref{fig:gym_comparison} shows the main comparison for \textit{Safe-control-gym} and \textit{Safety-gym}, respectively. We organize the analysis by environment group, followed by an ablation study and a cross-environment analysis of prediction error as the primary predictor of filter effectiveness.

\begin{table*}[!t]
\centering
\caption{Comparison of all methods (mean\,$\pm$\,std over 3 seeds). Best safety (violation rate) per environment is in \textbf{bold}. LQR uses symbolic dynamics and a single run.}
\label{tab:results}
\begin{tabular}{llrrrrr}
\toprule
\normalsize
Environment & Method & Return\,$\uparrow$ & Cost\,$\downarrow$ & Viol.~Rate\,$\downarrow$ & Seeds \\
\midrule
\textbf{CartPole Stab} & KCBF-SAC (ours) & $88.6\pm0.3$ & $9.6\pm1.2$ & $\mathbf{0.000\pm0.000}$ & 3 \\
 & LQR & $5.5$ & $5.5$ & $0.164$ & 1 \\
 & SAC & $88.5\pm0.6$ & $9.6\pm1.3$ & $0.230\pm0.017$ & 3 \\
 & SAC+Lagrangian & $34.0\pm33.4$ & $1.7\pm0.4$ & $0.204\pm0.002$ & 3 \\
 & SAC+Penalty & $54.4\pm38.1$ & $3.2\pm2.5$ & $0.218\pm0.055$ & 3 \\
\midrule
\textbf{CartPole Track} & KCBF-SAC (ours) & $97.0\pm1.5$ & $23.9\pm1.6$ & $\mathbf{0.000\pm0.000}$ & 3 \\
 & LQR & $30.4$ & $50.4$ & $0.463$ & 1 \\
 & SAC & $95.6\pm0.5$ & $28.4\pm1.8$ & $0.350\pm0.012$ & 3 \\
 & SAC+Lagrangian & $20.8\pm18.1$ & $3.4\pm1.8$ & $0.332\pm0.133$ & 3 \\
 & SAC+Penalty & $89.3\pm0.5$ & $2.5\pm0.3$ & $0.183\pm0.008$ & 3 \\
\midrule
\textbf{Quadrotor Stab} & KCBF-SAC (ours) & $195.5\pm1.6$ & $1.3\pm0.2$ & $\mathbf{0.004\pm0.001}$ & 3 \\
 & LQR & $72.4$ & $113.2$ & $0.640$ & 1 \\
 & SAC & $196.0\pm1.0$ & $1.6\pm0.2$ & $0.006\pm0.001$ & 3 \\
 & SAC+Lagrangian & $196.3\pm1.7$ & $1.2\pm0.2$ & $0.005\pm0.001$ & 3 \\
 & SAC+Penalty & $195.3\pm0.5$ & $0.8\pm0.1$ & $0.003\pm0.000$ & 3 \\
\midrule
\textbf{Quadrotor Track} & KCBF-SAC (ours) & $170.6\pm32.1$ & $32.3\pm31.8$ & $\mathbf{0.004\pm0.001}$ & 3 \\
 & LQR & $84.5$ & $108.3$ & $0.698$ & 1 \\
 & SAC & $195.8\pm0.2$ & $31.4\pm3.4$ & $0.125\pm0.014$ & 3 \\
 & SAC+Lagrangian & $194.3\pm0.7$ & $3.8\pm1.1$ & $0.015\pm0.004$ & 3 \\
 & SAC+Penalty & $192.7\pm1.1$ & $1.7\pm0.2$ & $0.007\pm0.001$ & 3 \\
\midrule
\textbf{HalfCheetah} & KCBF-SAC (ours) & $2290\pm22$ & $191\pm24$ & $\mathbf{0.828\pm0.008}$ & 3 \\
 & SAC & $7741\pm550$ & $979\pm1$ & $0.979\pm0.001$ & 3 \\
 & SAC+Lagrangian & $6781\pm1022$ & $977\pm1$ & $0.977\pm0.001$ & 3 \\
 & SAC+Penalty & $7267\pm756$ & $978\pm1$ & $0.978\pm0.001$ & 3 \\
\midrule
\textbf{Walker} & KCBF-SAC (ours) & $2607\pm202$ & $66\pm64$ & $0.370\pm0.226$ & 3 \\
 & SAC & $3803\pm292$ & $778\pm65$ & $0.778\pm0.065$ & 3 \\
 & SAC+Lagrangian & $2801\pm9$ & $33\pm15$ & $\mathbf{0.033\pm0.015}$ & 3 \\
 & SAC+Penalty & $2826\pm200$ & $131\pm64$ & $0.131\pm0.064$ & 3 \\
\bottomrule
\end{tabular}
\end{table*}

\subsection{CartPole: Zero-Violation Safe Control}

On both CartPole tasks, KCBF-SAC achieves a violation rate of exactly $0.000$ across all seeds while matching unconstrained SAC return: $88.6\pm0.3$ (vs.\ SAC $88.5\pm0.6$) on stabilization and $97.0\pm1.5$ (vs.\ SAC $95.6\pm0.5$) on tracking.
The final intervention rate on tracking is $\approx0.02\%$, indicating that the actor has internalized the constraint and the filter rarely needs to correct the nominal action.
The minimum barrier value across training, $h_{\min}\approx1.94$ (stab) and $1.81$ (track), confirms the safe set $\mathcal{C}$ is maintained with positive margin (see Figure~\ref{fig:barrier_h}).

In contrast, the constrained baselines fail to reliably satisfy the position constraint.
SAC+Lagrangian exhibits high return variance ($34.0\pm33.4$ on stabilization), a symptom of oscillating dual variables that overpenalize the policy before the Lagrange multiplier stabilizes.
SAC+Penalty achieves a violation rate of $18.3\%$ on tracking ($47.7\%$ relative reduction vs.\ SAC's $35.0\%$), but this improvement does not enforce pointwise constraints and leaves nearly one in five steps unsafe.

The success on CartPole is explained by the small projected residual margin: $\rho=9\times10^{-4}$.
Position is a direct state coordinate included in $\psi$, so the lifted barrier is exact, and the EDMD model predicts cart dynamics with negligible residual in the barrier direction.
The robust CBF condition~\eqref{eq:robust_kcbf_condition_intro} thus retains a large effective safety margin, and the filter certificate holds without slack activation.

\subsection{Quadrotor 2D: Composite Barrier and Tracking}

On quadrotor stabilization, all learning methods perform comparably (return $\approx195$--$196$), and the differences in violation rate are small.
The more informative result is quadrotor tracking, where KCBF-SAC reduces the violation rate from $12.55\%$ (SAC) to $0.40\%$, a \textbf{96.8\% reduction}. KCBF-SAC return is $170.6\pm32.1$ vs.\ unconstrained SAC $195.8\pm0.2$; the return variance is discussed below.

This result depends critically on the composite barrier~\eqref{eq:composite_barrier}, which is essentially a higher order CBF or HOCBF\cite{xiao2021high}.
The altitude constraint $y\geq y_{\min}$ has relative degree 2 with respect to vertical thrust: the standard lifted barrier gives $c^\top B\approx0$, making the CBF-QP degenerate.
Without the composite design, the effective control authority $\|a_{\mathrm{cbf}}\|$ collapses to $5.47\times10^{-5}$; the composite barrier restores it to $2.21\times10^{-3}$ (40$\times$ improvement), enabling the QP to produce meaningful corrections.

The return variance on tracking ($170.6\pm32.1$) reflects episodic slack activations during aggressive maneuvers when the QP becomes infeasible under actuator limits.
When slack is required, the certificate is formally weakened, but the average violation rate remains below $0.4\%$.
The residual margin is $\rho=1.3\times10^{-3}$, comparable to CartPole, confirming that Quadrotor dynamics are well-captured by EDMD once the barrier has the correct relative degree.

\begin{figure}[htpb]
    \centering
    \includegraphics[width=0.49\textwidth]{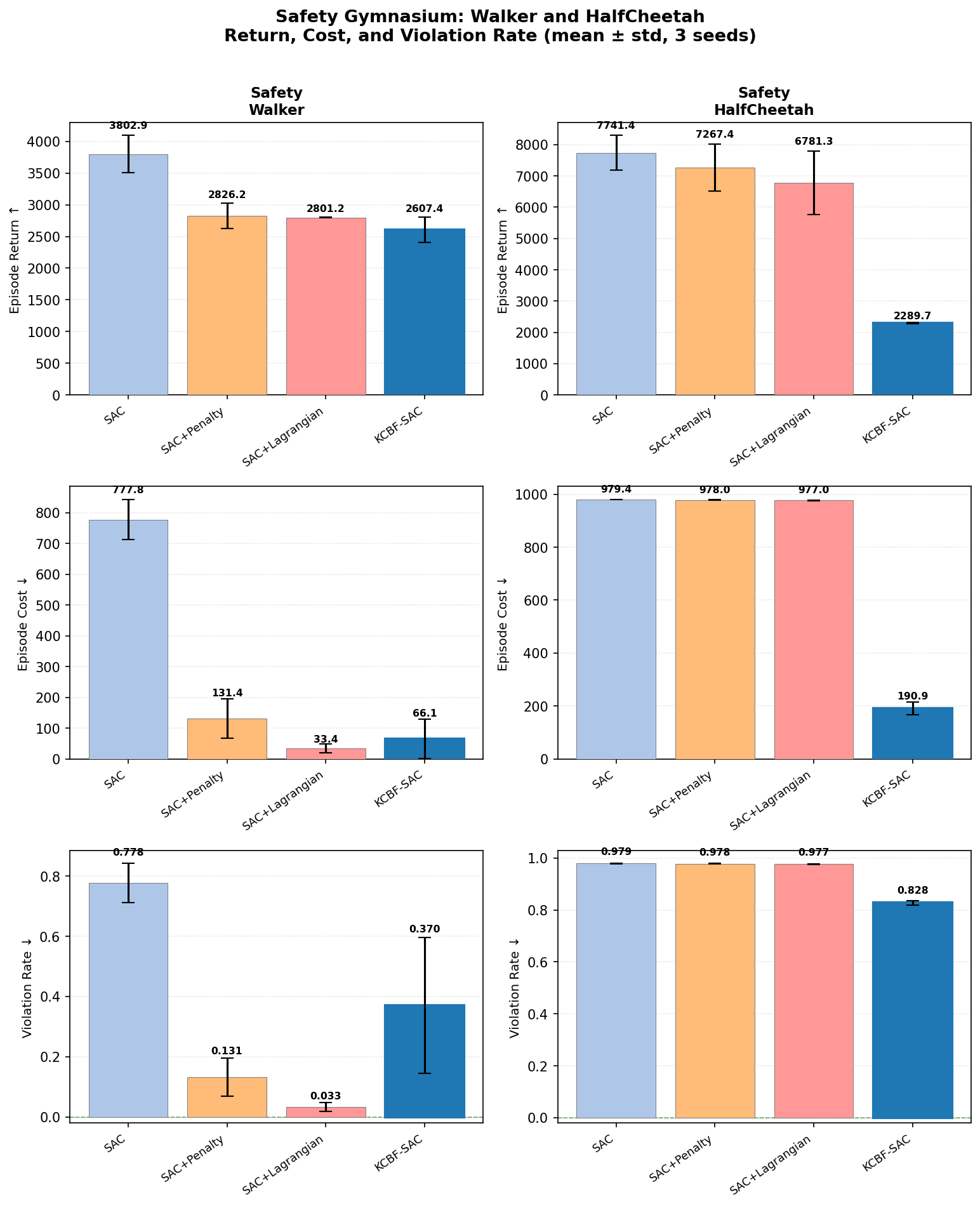}
    \caption{Safety Gymnasium results: episode return, cost, and violation rate for Walker and HalfCheetah. All methods fail to enforce the velocity constraint reliably; SAC+Lagrangian achieves the lowest violation on Walker (3.34\%) while KCBF-SAC is best on neither task due to large projected residual $\rho$. Error bars are mean\,$\pm$\,std over 3 seeds.}
    \label{fig:gym_comparison}
\end{figure}

\subsection{Safety Gymnasium: Exposing Structural Limitations}

\textbf{Walker}
KCBF-SAC reduces violation from $77.8\%$ (SAC) to $37.0\%$ while achieving return $2607\pm202$.
However, SAC+Lagrangian dominates on safety ($3.34\%$ violation) at comparable return ($2801\pm9$, lower variance), demonstrating that a distributional constraint method can outperform a step-wise CBF filter when the filter's structural assumptions are not met.

The high violation rate and large variance across seeds reflect two failure modes.
First, the projected residual $\rho=0.698$ is large relative to the typical barrier value $h_K(z_t)$.
When $h_K(z_t)<0$ (agent already unsafe) and $(1-\eta)|h_K(z_t)|>\rho$, the RHS of the robust CBF condition becomes negative and is trivially satisfied by any control input, disabling the filter entirely.
Even when $h_K(z_t)>0$, a large $\rho$ forces the predicted next barrier to exceed $\rho$, which may be infeasible under actuator limits and triggers slack activation.
Second, $h_{\min}=-7.81$ confirms that the agent regularly enters the unsafe region; the filter cannot provide a certificate once $h_K(z)<0$ (Figure~\ref{fig:barrier_h}).

\textbf{HalfCheetah}
KCBF-SAC reduces violations from $97.9\%$ (SAC) to $82.8\%$, but at a severe return cost ($7741\pm550\to2290\pm22$, $-70\%$).
The minimum barrier value $h_{\min}=-18.82$ indicates that the velocity constraint is deeply violated throughout training.
The residual margin $\rho=1.782$ exceeds typical barrier values at every step, making the CBF constraint trivially satisfied and the filter effectively inactive except in rare states with large $h_K$ headroom.

The fundamental cause is that ground-contact impacts produce sharp, non-smooth velocity transitions that the linear EDMD model cannot capture with $M=32$ RBFs; for the RBF dictionary class used here, $\rho$ does not decrease meaningfully with dictionary size in our experiments, though richer representations (e.g., deep Koopman lifting \cite{shi2022deep}) may help.
The return penalty arises because the CBF penalty term $\ell_{\mathrm{cbf}}$ in the actor loss~\eqref{eq:actor_loss} still discourages high-velocity actions even when the QP correction is negligible.

\begin{figure}[htpb]
    \centering
    \includegraphics[width=\columnwidth]{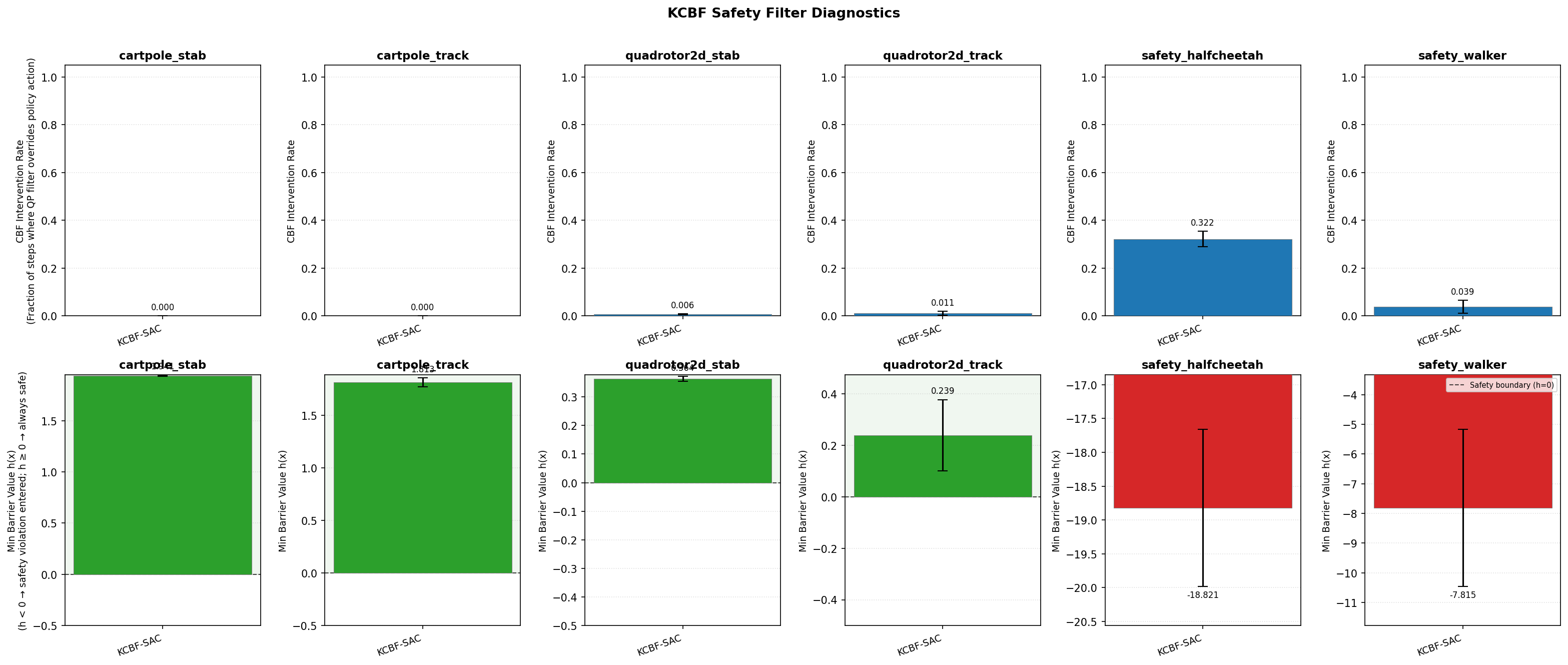}
    \caption{KCBF-SAC diagnostics per environment: intervention rate, slack rate, and minimum barrier value $h_{\min}$. CartPole and Quadrotor show near-zero slack and positive $h_{\min}$; Walker and HalfCheetah show high slack rates and negative $h_{\min}$, indicating loss of safety certificate.}
    \label{fig:diagnostics}
\end{figure}

\subsection{Barrier Value Evolution During Training}

Figure~\ref{fig:barrier_h} tracks the minimum barrier value $h_{\min}$ per episode over the course of training, averaged across 3 seeds, for all six environments under KCBF-SAC.

For CartPole (stabilization and tracking), $h_{\min}$ rises rapidly above zero within the first 20{,}000 steps and remains positive throughout training, confirming that the safe set $\mathcal{C}=\{z:h_K(z)\geq0\}$ is maintained from early in learning.
The smooth convergence is consistent with the small residual margin ($\rho\approx10^{-3}$) such that once the filter corrects the actor's actions, the barrier is not threatened again.

For Quadrotor 2D (both tasks), $h_{\min}$ is positive and stable throughout training, reaching $h_{\min}\approx0.24$--$0.36$, which confirms that the composite barrier~\eqref{eq:composite_barrier} provides effective altitude protection.
The occasional dips during tracking reflect transient slack activations during aggressive maneuvers, but $h_{\min}$ recovers.

For Safety Walker and HalfCheetah, $h_{\min}$ remains deeply negative throughout training ($h_{\min}\approx-7.8$ for Walker, $h_{\min}\approx-18.8$ for HalfCheetah), indicating that the KCBF filter never succeeds in enforcing the velocity constraint.
The barrier value does not improve over training, which corroborates the interpretation that the filter is structurally inactive (large $\rho$, trivially satisfied constraints) rather than transiently violated.
This training-time view complements the static $\rho$ diagnostic: both predict the same dichotomy between CartPole/Quadrotor (filter effective) and Gymnasium tasks (filter inactive).

\begin{figure*}[!t]
    \centering
    \includegraphics[width=\textwidth]{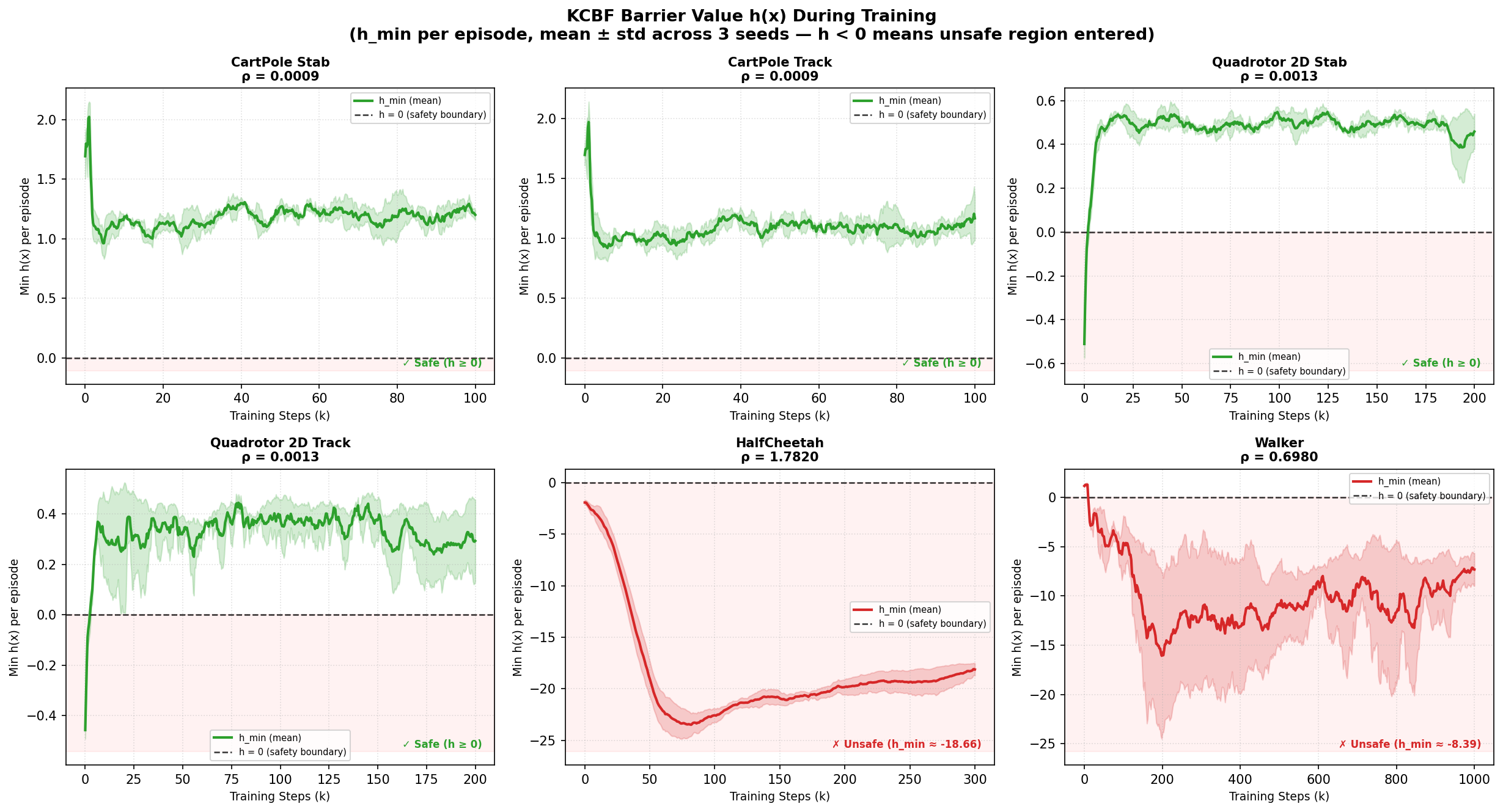}
    \caption{KCBF barrier value $h_{\min}$ per episode during training (mean\,$\pm$\,std, 3 seeds). CartPole and Quadrotor (green) maintain $h_{\min}>0$ from early training; Walker and HalfCheetah (red) remain at $h_{\min}\ll0$ throughout, confirming the filter is structurally inactive for these environments. Dashed line marks the safety boundary $h=0$.}
    \label{fig:barrier_h}
\end{figure*}

\subsection{Ablation: CBF Discount parameter on Walker}

We ablate the CBF discount $\eta$ and slack weight $\lambda_\xi$ on Safety Walker, fixing $\lambda_\xi=10{,}000$ and varying $\eta\in\{0.5,0.7,0.9\}$.
Configurations $\eta=0.7$ and $\eta=0.5$ were trained for 1M steps with a single seed; Figure~\ref{fig:ablation} shows convergence alongside the default ($\eta=0.9$, 3 seeds) and SAC+Lagrangian.

At 1M steps, the three configurations reveal qualitatively distinct failure modes.
Default $\eta=0.9$ achieves the best balance: return $2607\pm202$ and violation $37.0\%$.
\textit{Policy collapse} ($\eta=0.7$): the tighter CBF constraint causes the QP to override the actor at $85\%$ of steps; the actor fails to learn effective locomotion and produces a bimodal policy, i.e., some episodes maintain $h_K>0$ (short, conservative runs) while others fail catastrophically. Yielding a mean return of $\approx26$ and a misleadingly low violation rate of $6.7\%$ that reflects an inactive policy rather than genuine safety.
\textit{Constraint infeasibility} ($\eta=0.5$): the tighter constraint becomes infeasible under actuator limits at most steps, causing pervasive slack activation and an effective intervention rate of only $6.2\%$. The actor learns freely and achieves return $3405$, but with $84.5\%$ violations.

Among the three tested values, $\eta=0.9$ is the only stable operating point for Walker's residual regime; higher values of $\eta$ were not evaluated.
Lower $\eta$ tightens the barrier, but when $\rho$ is already large, the constraint either overrides the actor into collapse or becomes infeasible and is bypassed entirely.
Early-training snapshots (150k steps) were misleading in the sense that Walker's policy had not yet learned fast locomotion, so all configurations appeared safe and divergence emerges only after the agent learns high-velocity gaits.

\begin{figure*}[htpb]
    \centering
    \includegraphics[width=\textwidth]{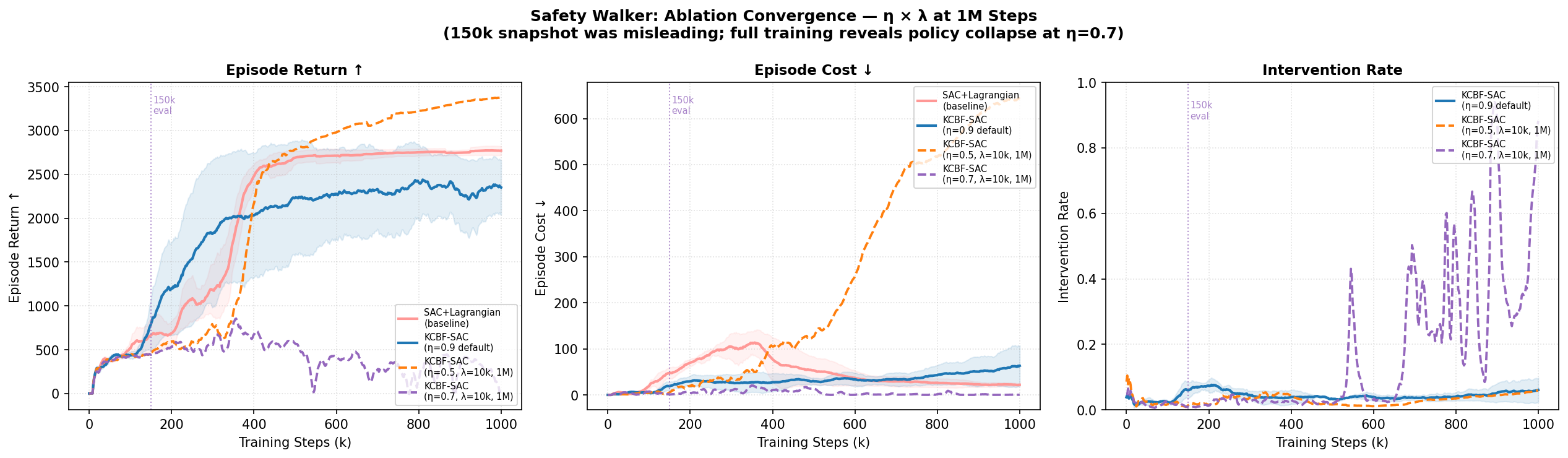}
    \caption{Safety Walker ablation: return, cost, and intervention rate over 1M training steps. $\eta=0.7$ ($\lambda_\xi=10{,}000$) suffers policy collapse (mean return $\approx26$, 85\% intervention rate, bimodal per-episode behavior). $\eta=0.5$ achieves high return (3405) but near-total safety failure (84.5\% violations) via pervasive slack activation. Among tested values, $\eta=0.9$ is the only stable operating point for Walker's residual regime ($\rho=0.698$). Single seed for ablation variants; 3 seeds (shaded) for baselines.}
    \label{fig:ablation}
\end{figure*}

\subsection{Koopman Residual parameter as a Filter-Effectiveness Predictor}

\begin{figure*}[htpb]
    \centering
    \includegraphics[width=\textwidth]{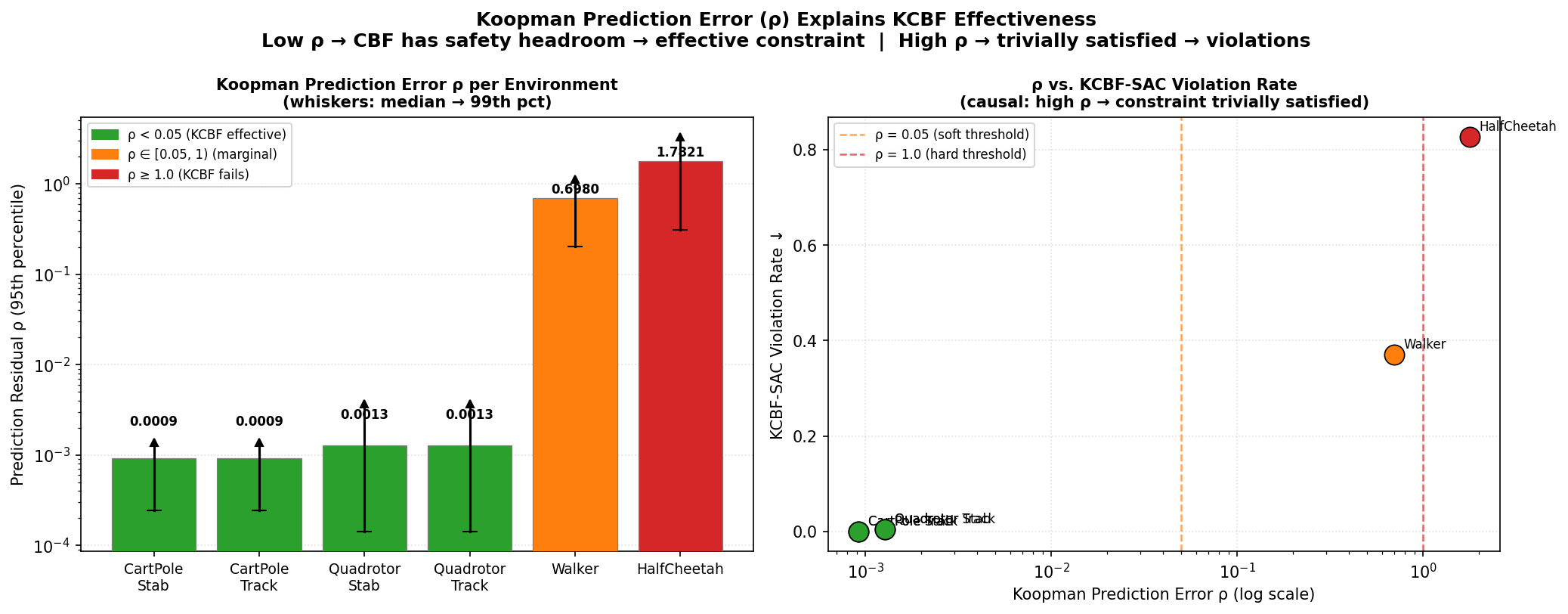}
    \caption{Left: projected residual margin $\rho$ (95th-percentile $|c^\top r_i|$) per environment on log scale; whiskers show median-to-99th-percentile spread. Right: $\rho$ vs.\ KCBF-SAC violation rate, revealing a near-monotone relationship spanning more than three orders of magnitude ($\rho=9\times10^{-4}$ for CartPole to $\rho=1.78$ for HalfCheetah).}
    \label{fig:rho}
\end{figure*}

Figure~\ref{fig:rho} reveals a clear relationship between the projected residual $\rho$ and KCBF-SAC effectiveness.
CartPole ($\rho=9\times10^{-4}$, violation $0\%$) and Quadrotor ($\rho=1.3\times10^{-3}$, violation $0.4\%$) lie in a regime where the robust CBF constraint retains meaningful headroom.
Walker ($\rho=0.698$) and HalfCheetah ($\rho=1.782$) lie in a regime where $\rho$ alone saturates the typical barrier value, making the constraint trivially satisfied.

This relationship follows directly from the structure of the robust CBF condition.
The constraint RHS is $(1-\eta)h_K(z_t)+\rho$.
There are two distinct failure regimes.
When $h_K(z_t)>0$ but $\rho$ is large, the predicted next barrier must exceed $(1-\eta)h_K(z_t)+\rho\geq\rho$; satisfying this requirement may be infeasible under actuator limits, causing slack activation and certificate loss.
When $h_K(z_t)<0$ (agent already unsafe) and $(1-\eta)|h_K(z_t)|>\rho$, the RHS is negative and every control input trivially satisfies the constraint, completely disabling the filter.
Both regimes are active on Walker and HalfCheetah, as confirmed by the large slack rates and deeply negative $h_{\min}$ in Figure~\ref{fig:diagnostics}.
In practice, $\rho$ is computable before training as a static EDMD quality diagnostic; both Walker ($\rho=0.698$) and HalfCheetah ($\rho=1.782$) exhibit these failure modes, while CartPole and Quadrotor ($\rho\leq10^{-3}$) do not.

\subsection{Limitations}

The experiments expose six structural limitations of the proposed method.

\textbf{Relative degree}:
The CBF-QP requires the lifted barrier to have relative degree 1 with respect to the action ($c^\top B\neq 0$).
The composite barrier~\eqref{eq:composite_barrier} resolves the altitude constraint for Quadrotor 2D, but it is a case-by-case fix.
Discrete-time CBFs~\cite{agrawal2017discrete} for higher-relative-degree systems offer a systematic extension.

\textbf{Irreducible prediction error}:
When contact dynamics produce non-smooth state transitions (HalfCheetah, $\rho=1.782$), no practical dictionary size $M$ reduces $\rho$ to a regime where the filter is effective.
Extensions to multi-step Koopman prediction or locally nonlinear lifting are needed.

\textbf{Policy collapse under over-constraint}:
Tight CBF parameters ($\eta=0.7$, Walker) cause the QP to dominate actor updates, producing a degenerate policy.
This failure mode is absent from Lagrangian methods, which enforce constraints through a differentiable dual objective rather than hard action correction.

\textbf{One-step certificate}:
The CBF condition is enforced step-wise; no multi-step planning horizon is used.
Trajectories can accumulate in regions where $h_K$ is small even if each individual step satisfies the constraint.

\textbf{Static barrier design}:
Barrier functions are hand-specified per environment.
Joint learning of the Koopman lifting and the barrier certificate~\cite{zinage2023neural} would improve generality and reduce design effort.

\textbf{Slack degrades the certificate}:
When the QP is infeasible without slack ($\xi>0$), the filter no longer enforces a formal CBF condition.
Slack rate should be monitored as a safety diagnostic; environments with high slack rate (see Figure~\ref{fig:diagnostics}) do not provide formal step-wise guarantees even when violations appear low.

\section{Conclusion}
\label{sec:conclusion}

We introduced Robust Koopman-CBF SAC, a safety-filtered actor-critic framework that lifts nonlinear dynamics to an approximately linear Koopman predictor, enforces an affine CBF constraint through a lightweight QP safety filter, and accounts for finite-dimensional approximation error via a projected residual margin $\rho$.
The framework integrates with SAC through executed-action critic learning and an actor-side CBF feasibility penalty, eliminating the policy-filter mismatch that arises when safety corrections are applied post-hoc to an otherwise unconstrained actor.

On CartPole stabilization and tracking, the method achieves zero constraint violations across all seeds while matching unconstrained SAC returns, demonstrating that a Koopman-CBF filter can provide formal pointwise safety with no performance cost when the lifted model is accurate and the barrier has the correct relative degree.
The 96.8\% violation reduction on Quadrotor tracking further shows that structural barrier design (specifically, compositing an RD-2 altitude constraint into an RD-1 form) is as important as model quality.
Together these results confirm the method's promise for low-to-medium-dimensional robotic control tasks where EDMD quality is high.

On Safety Gymnasium locomotion, the results are more nuanced.
KCBF-SAC reduces violations on Walker (from $77.8\%$ to $37.0\%$) but is outperformed by SAC+Lagrangian on safety ($3.34\%$), which is notable because the Lagrangian baseline enforces only an expected cumulative cost constraint rather than a step-wise barrier condition.
This counter-intuitive outcome reflects a fundamental limitation of the first-order Koopman-CBF approach: with $\rho=0.698$, contact dynamics dominate prediction error and the robust CBF condition alternates between trivially satisfied (when $h_K<0$) and infeasible under actuator limits (when $h_K>0$).
The projected residual $\rho$, measurable before training, serves as a reliable diagnostic that predicts filter effectiveness across more than three orders of magnitude, from CartPole ($\rho\approx10^{-3}$, zero violations) to HalfCheetah ($\rho=1.782$, $82.8\%$ violations).

These findings motivate several extensions.
Discrete-time Koopman-CBFs for higher-relative-degree systems~\cite{agrawal2017discrete} can address relative-degree limitations without case-by-case composite barrier construction.
Multi-step Koopman prediction horizons can reduce the gap between the 1-step certificate and multi-step trajectory safety.
Online Koopman model adaptation would allow $\rho$ to decrease as the agent explores safer regions of the state space.
Joint learning of the lifting map and barrier certificate~\cite{zinage2023neural} would remove the dependence on hand-designed barriers, and conformal calibration~\cite{angelopoulos2023conformal} offers a principled approach to tightening $\rho$ with distribution-free coverage guarantees.
Together these directions point toward Koopman-CBF safety filters that remain effective in the contact-rich, high-dimensional settings where current first-order approaches reach their structural limits.

\bibliographystyle{IEEEtran}
\bibliography{References}

\end{document}